%% file: main.tex
\documentclass[11pt]{article}
\usepackage{gal-basic}

\title{Consider the Alternatives: Navigating Fairness-Accuracy Tradeoffs via Disqualification}

\author{Guy N.\ Rothblum\thanks{This project has received funding from the European Research Council (ERC) under the European Union’s Horizon 2020 research and innovation programme (grant agreement No. 819702), from the Israel Science Foundation (grant number 5219/17), from the U.S.-Israel Binational Science Foundation (grant 2018102), and from the Simons Foundation Collaboration on the Theory of Algorithmic Fairness. Part of this work was done while the author was visiting Microsoft Research.}
\\Weizmann Institute of Science\\\texttt{rothblum@alum.mit.edu}
\and
Gal Yona\thanks{This project has received funding from the European Research Council (ERC) under the European Union’s Horizon 2020 research and innovation programme (grant agreement No. 819702), from the Israel Science Foundation (grant number 5219/17) and from the Simons Foundation Collaboration on the Theory of Algorithmic Fairness. This research was partially supported by the Israeli Council for Higher Education (CHE) via the Weizmann Data Science Research Center, and by a research grant from the Estate of Tully and Michele Plesser.}
\\Weizmann Institute of Science\\\texttt{gal.yona@weizmann.ac.il}}

\begin{document}

\maketitle

\input{abstract}

\section{Introduction}
\label{sec:intro}



Underlying the study of \emph{algorithmic fairness} \cite{dwork2012fairness, hardt2016equality, kusner2017counterfactual, kearns2018preventing, hebert2017calibration} in the context of supervised learning is the fundamental tension between different fairness desiderata and some other desired property, such as accurate predictions \cite{kleinberg2016inherent, chouldechova2017fair, chen2018my, pleiss2017fairness}. This tension exists for a broad array of non-discrimination requirements. For example, demographic parity (see e.g. \cite{dwork2012fairness}) requires that the predicted labels are independent of group membership -- a requirement that stands in direct tension with accuracy when when the base rates (expected labels) between the groups differ. In other cases, the tension is more subtle but still exists. For example, the notion of equal opportunity \cite{hardt2016equality} requires that the error rates of the classifier (false positives and/or false negatives) be equal across groups. While a perfect classifier that never errs does satisfy this requirement, in settings where perfect predictions are impossible (e.g. due to missing data or inherent uncertainty), satisfying error rate parity can stand in direct conflict with maximizing accuracy.\footnote{For example, suppose that the only data feature is group membership (in $S$ or in $T$), that 80\% of $T$ have label 1 and 20\% of $S$ have label 1. For non-trivial accuracy, we need to predict a 1-label with higher probability for members of $S$ than for members of $T$, but this will result in a disparity in false negatives and false positives between the two groups. }  Importantly, in practice the tension between fairness and accuracy may (and does) arise even when the learner is well-intentioned, e.g. because of missing data or constraints on the complexity of the learned predictor (which are necessary to ensure generalization).



\begin{figure*}[ht]
\centering
\captionsetup{width=0.9\linewidth}

\includegraphics[width=1.0\linewidth]{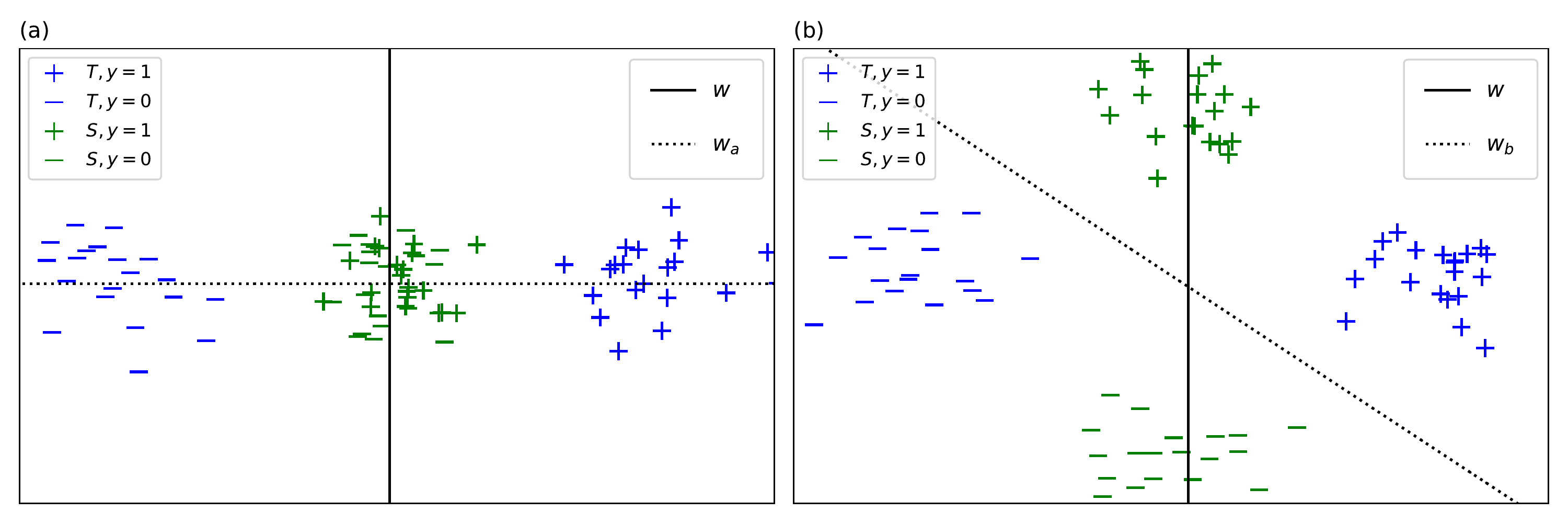}
\caption{In both cases, the classifier $h_w(x) = \text{sign}\langle w, x \rangle$ does not satisfy error rate parity, whereas the classifiers $h_{w_a}(x) = \text{sign}\langle w_a, x \rangle$ (left) and $h_{w_b}(x) = \text{sign}\langle w_b, x \rangle$ (right) do.
 }
\label{fig:balance-issue}
\end{figure*}

When fairness and accuracy are in conflict, as they often are, how can we reason about the possible trade-offs? If we require that error rates be balanced, this might preclude reasonable accuracy. We could opt, instead, to bound the error-rate disparity \cite{zafar2017fairness}, but what level of disparity would be appropriate? These considerations go well beyond the technical aspects of the problem at hand, and the answers must be informed by ethical, societal and legal considerations. In particular, any criterion that only considers a predictor in light of its error rate imbalance is inherently limited. Drawing inspiration from the legal literature, the following discussion of Disparate Impact doctrine from \cite{barocas2016big}, p694, is instructive:
\begin{quote}
    ``Disparate impact is not concerned with the intent or motive for a
policy; where it applies, the doctrine first asks whether there is a disparate
impact on members of a protected class, then whether there is some business
justification for that impact, and finally, \emph{whether there were less discriminatory
means of achieving the same result}''
\end{quote}

Consider the tension between accuracy and fairness in light of this doctrine: we have disparate impact in the form of imbalanced error rates (or another measure of unfairness). Accuracy requirements might provide a business justification for this imbalance. However, {\em we must also consider the alternatives}, and whether there are less discriminatory means of achieving the same result.

\vspace{1em}

We demonstrate this point through a simple example. We consider a simple setup with two groups $S$ and $T$ towards which we want to guarantee fair treatment.\footnote{While we focus on the two-group setup, our approach naturally lends itself to a ``multi-group'' extension \cite{hebert2017calibration, kearns2018preventing}, where instead of only two groups we consider a much a larger collection of groups; see the discussion in Section \ref{section:discuss}.} Consider the two scenarios described in Figure \ref{fig:balance-issue}: individuals from $T$ are denoted in blue, and individuals from $S$ are denoted in green. The learning task at hand is finding a low-error hyperplane. In both cases the data for $T$ is perfectly separable, and the hyperplane $w$, which is tailored for $T$, has no false positives or false negatives on $T$. On the set $S$, however, the hyperplane $w$ performs poorly (in both scenarios): half of its classifications will be false positives or false negatives. Given the disparity in error rates between $S$ and $T$, should we rule out the predictor $w$ for being discriminatory? Here, the two examples may diverge. In scenario $(b)$ the hyperplane $w$ is clearly and blatantly unfair: the hyperplane $w_b$, which takes $S$ into account, has zero error-rate disparity, and is no less accurate than $w$ (indeed, its loss is smaller).
In scenario $(a)$, on the other hand, the data of $S$ is inseparable and there is no clear alternative to $w$. The only clear alternative with balanced error rates is $w_a$, which has huge error. While error rate disparity is an intuitively appealing measure, on its own it does not allow us to make distinctions between very different scenarios, such as $(a)$ and $(b)$.\footnote{We emphasize that we do not assert  the hyperplane $w$ is a ``fair'' classifier in scenario $(a)$: this is a task-specific and context-dependent question. Our main point is that studying the set of alternative classifiers is instructive, and points to the glaring unfairness of using $w$ in scenario $(b)$. Moreover, studying these alternatives can also be instructive in scenario $(a)$: it may lead to the conclusion that we need to collect new data features or consider a richer hypothesis class.}

\paragraph{Exact disqualification.} The example discussed above suggests a natural first step towards formalizing the disparate impact doctrine: a classifier should be ``disqualified'' if there is an alternative classifier with similar accuracy that is more fair, as was the case for the hyperplane $w$ in scenario $(b)$ above. Importantly, this requires that we specify a set $\H$ of alternatives to be considered. Most naturally, $\H$ could be the hypothesis class used by a learning procedure. However, there is an important subtlety here: what do we mean by an alternative classifier being ``more fair''? We might define the amount of unfairness in absolute value, e.g. the absolute value of the error disparity between the two groups, but this is problematic. Consider two equally accurate classifiers $h_S$ and $h_T$, where $h_g$ denotes a classifier that has an error disparity in favour of group $g \in \set{S,T}$. When measuring unfairness in absolute value, neither classifier disqualifies the other! But this is not what we want, as intuitively, neither classifier should be ``kosher'' to use: For example, from the perspective of $S$, it is unjustified to use $h_T$ (which has a high error rate on $S$) when there is an equally-accurate classifier that $S$ prefers. Similarly, $T$ could make a similar argument against $h_S$. This discussion highlights that in order to prevent unfair treatment towards both groups, our notion of fairness must carefully take into account the \emph{directionality} of the fairness violations. 


\paragraph{Beyond Exact Disqualification.}
Exact disqualification gives us a framework to compare a given classifier to a set of available alternatives, but it does not fully resolve the question of navigating fairness and accuracy tradeoffs because it leaves many classifiers incomparable. For example,
if relative to a classifier $h$, a classifier $h'$  improves fairness by $\tau$ but hurts accuracy by $\tau'$, neither classifier disqualifies the other. 
Choosing between such classifiers implicitly requires comparing a fairness gain of $\tau$ to an accuracy loss of $\tau'$. This comparison is not straightforward, since these quantities are typically not measured in the same ``units'' - for example, accuracy is usually averaged over the entire population whereas fairness typically quantifies some across-group disparity. This difficulty is further exacerbated by the fact that both fairness and accuracy are in general two abstract concepts, that can be operationalized in different ways (e.g. by choosing different metrics based on the task in question).

\vspace{1em}

 We address these challenges by quantifying the relative importance of fairness vs. accuracy. In particular, our  notion of $\gamma-$disqualification stipulates that a classifier should be disqualified if there is a fairer alternative that does not degrade accuracy {\em ``too much''}. This is more restrictive than exact disqualification: we enlarge the set of disqualifying alternatives, by including some that do degrade accuracy, so long as the accuracy degradation is not ``too much'' larger than the improvement in fairness. The notion of ``too much'' is quantified using a parameter $\gamma \geq 0$. Importantly, our objective is for $\gamma$ to reflect the acceptable tradeoffs between fairness and accuracy in a way that is independent of the specific metrics chosen to operationalize these concepts.  As discussed above, this requires specifying an appropriate normalization to first bring fairness and accuracy to the same ``units'', so the parameter $\gamma$ can meaningfully specify acceptable accuracy-fairness tradeoffs.  This is a crucial ingredient (and contribution) of our framework. 
The larger $\gamma$ is, the more weight we give to any degradation in the classifier's accuracy, and the less restrictive the requirement becomes:

\begin{center}
    1 unit of accuracy $\equiv \,\, \gamma$ units of fairness \\
    1 unit of fairness $\equiv \,\, 1/\gamma$ units of accuracy
\end{center}

In particular, setting $\gamma = \infty$ recovers exact disqualification, whereas $\gamma=0$ means that we require that the classifier be optimally fair within $\H$. For example, if the fairness notion we use is parity of error rates and the class $\H$ includes a balanced classifier (e.g. the constant classifier), then setting $\gamma=0$ is equivalent to requiring error rate parity. The range of values $\gamma \in (0,\infty)$ specifies a range of acceptable fairness-accuracy trade-offs.  We refer to the resulting requirement in short as $\gamma$-Fairness. Our framework can be instantiated with different loss functions, capturing different measures of accuracy, and their trade-offs with different measures of (un)fairness.

\subsection{Related work}

\paragraph{Other approaches to formalizing ``disparate impact''.} We defer the full discussion of the different avenues by which the fairness community has proceeded in light of the basic tradeoffs between fairness and accuracy to Section \ref{section:discuss}, and here discuss how disqualification differs from perhaps the other natural approach to addressing the ``business necessity'' clause of the disparate impact doctrine, which is to maximize fairness subject to an accuracy constraint \cite{zafar2017fairness}. They key difference is that this approach results in a model whose accuracy is not far from the accuracy of the optimal model, whereas such a property is not guaranteed in our framework, since it could be the case that all the non-disqualified models are significantly less accurate than the optimal one.  However, we argue that this property is not necessarily desirable. In particular, the fact that the eventual model should be ``not significantly less accurate than the optimal model'' encodes a normative assumption, that may not be valid in all cases.\footnote{To see this, consider an extreme case in which the data is highly biased, such that all high-accuracy models are in reality extremely unfair. In this situation, it’s not necessarily true that we want to insist on returning an accurate model (in fact, if we sufficiently value fairness, we might insist on returning a model that is fair and hence with very low accuracy).} Our framework handles such subtleties, because by definition,  a model with significantly lower accuracy will only be used if (i) there are significant trade-offs between fairness and accuracy, and (ii) fairness is deemed significantly more important than accuracy (as portrayed in the value of $\gamma$). 

\paragraph{Disqualification vs multi-objective optimization (MOO).}

In spirit, exact disqualification resembles the well studied notion of Pareto efficiency (a classifier should be disqualified if it is Pareto dominated by another classifier in $\H$); similarly, $\gamma$-disqualification resembles further restricting the set of Pareto efficient solutions via an a-priori preference method in which additional information is elicited from a decision-maker (see Chapter 3 in \cite{hwang2012multiple} for a detailed overview). This is similar in spirit to the value of $\gamma$ in our work. However, while similar in motivation, on a technical level the concepts of disqualification and Pareto efficiency are distinct. To see this, recall that in our framework the groups $S$ and $T$ are considered symmetric, in the sense that unjustified preferential treatment to either group is considered unfair. This means that ``unfairness`` cannot be measured in absolute value: for example, two equally accurate classifiers $h_S$ and $h_T$ (where $h_g$ is unfair towards group $g \in \set{S,T}$) may both be on the Pareto frontier. Specifically, in our work the direction of the fairness violation that is pertinent depends on the classifier in question: This breaks the analogy to Pareto efficiency, that like other concepts in MOO, is defined and evaluated w.r.t a \emph{fixed} set of objective functions.\footnote{However, our treatment of fair risk minimization in Section \ref{section:erm} reveals that while the two solution concepts are fundamentally distinct, we can use the established work on Pareto efficiency to derive optimal $\gamma$-fair classifiers.} As a special case, this also explains how our framework is different from the widely used practice of fairness-regularized risk minimization \cite{kamishima2012fairness, zafar2017fairness, donini2018empirical}. Optimizing a regularized objective (a standard loss term combined with a fairness loss term, scaled by some parameter $\lambda$) can be viewed as scalarizing a MOO problem, so the question of what exactly is the fairness loss remains. Thus, for similar arguments to those made above, it will not be the case that minimizing a fairness regularized objective will yield a valid solution under our framework.

A second important  distinction is that in our work we aim to formalize a framework where a single value $\gamma$ can capture the trade-offs deemed acceptable between fairness and accuracy; specifically, this value should not depend on the metrics used to quantify what fairness and accuracy mean for a particular task. This is important as indeed fairness and accuracy can mean different things in different contexts, and requires setting up principled translations between units of fairness and units of accuracy, which serves as a fundamental part (and contribution) of our work.


\subsection{Our contributions}

Our primary contributions are:

\begin{enumerate}
    \item \textbf{Formalizing $\gamma$-disqualification}, a flexible framework for navigating the space of fairness-accuracy tradeoffs with respect to a benchmark class $\H$. The framework can be instantiated with (potentially different) loss functions for measuring accuracy and for measuring disparity in predictions, and is parametrized by a ``scaling function'' that compares differences in accuracy to differences in disparities.  
    
    \item \textbf{Properly scaling fairness and accuracy.} 
    The scaling function is an important ingredient in our framework, and we develop a methodology for choosing it. One aspect we highlight is the minimal level  $\gamma$ for which the Bayes optimal predictor is not $\gamma$-disqualified by any other classifier. When possible, we aim to select the scaling function in a way that ``anchors'' this value at $\gamma = 1$. We instantiate this approach for two commonly used accuracy measures. First, for the squared loss, we put forward a natural scaling function, prove that it satisfies the above requirement, and also show that it has several desirable properties in the regime $\gamma < 1$.  Second, we consider the 0/1 loss and show that it behaves quite differently: there is no ``reasonable'' scaling function that guarantees the above requirement. These ``case studies'' are valuable  for enabling deployment of our framework, and they also illuminate how different choices of loss functions can quantitatively affect the tension between fairness and accuracy.

    \item \textbf{Fair risk minimization.} We present an algorithm that, given a dataset of labeled examples and the parameter $\gamma$, finds an approximately optimal (most accurate) predictor that is not $\gamma-$disqualified by another classifier in a given hypothesis class $\H$. The algorithm is stated as a reduction to the task of approximating the Pareto frontier of $\H$, where the latter task is well studied, and can be accomplished efficiently for simple classes such as linear regression. 

\end{enumerate}

\paragraph{Organization.} The rest of this manuscript is organized as follows. In Section \ref{section:pareto} we formalize the $\gamma-$disqualification framework. Section \ref{section:scaling} discusses the question of determining the appropriate ``translation'' between accuracy and fairness. In Section \ref{section:erm} we study fair risk minimization. Finally, Section \ref{section:discuss} discusses example applications, an extension to a ``multi-group`` setup and further related work.  Full proofs are deferred to the appendix.

\section{Disqualification}
\label{section:pareto}
\paragraph{Setup.} We use $\X$ to denote the feature space and $A \in \set{0,1}$ group membership, and consider randomized classifiers $h:\X\times A\to\hat{\Y}=[0,1]$, where $h(x,g)$ is the probability that $h$ will predict a label $1$ for an individual with features $x$ from group $g\in\set{S,T}$. We use $\ell_A: Y \times \hat{Y} \to [0,1]$ to denote the loss we use to measure global accuracy and $\ell_B: Y \times \hat{Y} \to [0,1]$ to denote the loss we use to measure disparities between groups. 

\paragraph{Quantifying unfairness.}

We say that a classifier is \emph{loss balanced} (for the positive class) w.r.t $S$ and $T$ if the average loss $\ell_B$ of individuals with $y=1$ is the same across $S$ and $T$. We quantify how far a given classifier is from satisfying this requirement, in a specific direction (e.g. how worse-off are the members of $S$ are, relative to the members of $T$) via the notion of \emph{loss imbalance}:

\begin{definition}[Loss imbalance]
\label{def:loss-imb}
Fix two groups $S,T$ and a loss $\ell_B$. The \emph{loss imbalance} (for the positive class) of a classifier $h: \X \to [0,1]$ in the direction $T \to S$ is

\begin{equation*}
    \text{dLossImb}(h; \, T \to S, \ell_B) = \psi(h,S) - \psi(h,T)
\end{equation*}

where $\psi(h,C) = \E_{x,  g, y \sim \P}[\ell(h(x), y)  \, \vert \, y=1, g = C]$.
\end{definition}

We note that our notion of loss imbalance (Definition \ref{def:loss-imb}) can be viewed as quantifying unfairness, where the notion of fairness generalizes many common existing group fairness definitions. In particular, when $\ell_B$ is the (expected) 0/1 loss, it recovers both  
\emph{Balance for the positive class} \cite{kleinberg2016inherent} and \emph{Equal Opportunity} \cite{hardt2016equality}, depending on whether the classifiers in questions are binary or randomized. See Appendix \ref{appendix:loss-imbalance} for details. For the rest of this paper we consider randomized classifiers and think of $\ell_B$ as the expected 0/1 loss.

\paragraph{Formalizing $\gamma$-disqualification.}

We want to say that a classifier $h'$ disqualifies a second  classifier $h$ if the former improves the loss imbalance (as measured w.r.t  $\ell_{B}$) over the latter, without hurting global accuracy (as measured w.r.t $\ell_{A}$) too much --  with the exact tradeoff specified by the parameter $\gamma$. However, as we discuss in Section \ref{section:scaling}, naively comparing the difference in the two quantities is an ``apples to oranges'' comparison. We address this by applying a scaling function to the difference in accuracy before comparing it to the difference in imbalance.  In principle, the scaling is allowed to depend on $\gamma$, so the scaling function is a mapping $f$ from $\R \times (0,\infty)$ to $\R$. For a fixed value $\gamma \in (0, \infty)$, we use $f_\gamma$ to denote the resulting function $f_\gamma: \R \to \R$.

\begin{definition*}[$\gamma$-Disqualification]
A classifier $h'$  $\gamma$-disqualifies a classifier $h$ w.r.t losses $\ell_A$ and $\ell_B$ and a scaling function $f$ if
\begin{equation}
\label{eqn:intro:disq}
    dLossImb(h;\ell_{B})-dLossImb(h';\ell_{B})> f_{\gamma}(\max\set{\ell_{A}(h')-\ell_{A}(h),\, \, 0})
\end{equation}

where $\mathit{dLossImb}(\cdot)$ is computed in the direction for which $dLossImb(h;\ell_{B}) \geq 0$.

\end{definition*}

Note that to assess whether $h'$ disqualifies $h$, we first take into account the direction of the fairness violation of the original classifier, $h$. This is designed to ensure that if $h$ is unfair e.g. because it favours the members of $T$ over $S$ (i.e. the imbalance of $h$ is positive in the direction $T \to S$; see Definition \ref{def:loss-imb}), only a classifier that improves the imbalance \emph{in this direction} can disqualify $h$.


\paragraph{``Consider the alternatives'': $\gamma$-Fairness.} Finally, we refer to the requirement that 
 a classifier not be disqualified by any alternative in a class $\H$ of alternative classifiers as $\gamma$-Fairness:

\begin{definition*}[$\gamma$-Fairness] Fixing loss functions $\ell_A$ and $\ell_B$ and a scaling function $f_\gamma$, we say that a classifier $h$ satisfies $\gamma$-fairness w.r.t  $\H\subseteq\hat{Y}^{X\times A}$ if no classifier $h'\in \H$ $\gamma$-disqualifies $h$ w.r.t $\ell_A$, $\ell_B$ and $f_\gamma$. When $\H$ is unconstrained, $\H = \hat{Y}^{X\times A}$, we simply say that $h$ is $\gamma$-fair.
\end{definition*}

\section{Specifying the scaling function}
\label{section:scaling}
Given losses $\ell_A$ and $\ell_B$ for measuring accuracy and fairness (respectively), how should we choose the scaling function that ``translates'' units of accuracy into units of fairness? As a warm-up, we revisit the simple examples of Figure \ref{fig:balance-issue}. Recall that there we were concerned with binary classifiers, and imbalance was measured as the disparity in \TPR~ across $S$ and $T$ (so $\ell_{A}$ and $\ell_{B}$ are both the 0/1 loss). Relative to $h_w$, the classifier $h_{w_{a}}$ improves fairness by $0.5$ and hurts accuracy by $0.25$ -- but how should we think about comparing between these numbers?  One natural perspective is to consider the (worst-case) impact a single individual can have on both accuracy and fairness. In this case, even though $\ell_{A}=\ell_{B}$, the (potential) contribution of a single individual to the change in  unfairness is much greater than it is to the change in accuracy loss, since the latter is measured globally over the entire population whereas fairness conditions on both group membership and label. This means that at the very minimum, the scaling function we choose must take this into account by appropriately ``up-weighting''  accuracy  before comparing it to the fairness. 
 
 \vspace{1em}
 
 To make this intuition more precise (and to generalize to arbitrary choices of $\ell_A$ and $\ell_B$), consider the following hypothetical question: How much can a certain improvement in accuracy -- resulting from changing the prediction of a single individual -- impact unfairness, in the worst case? To make this concrete, consider some classifier $h$ and individual $x$, and define a new classifier $h'$ as follows:
 
 \begin{equation*}
     h'(\tilde{x})=\begin{cases}
h^{\star}(\tilde{x}) & \tilde{x}=x\\
h(\tilde{x}) & \tilde{x}\neq x
\end{cases}
 \end{equation*}
 
 That is, $h'$ is identical to $h$, except we switched the prediction of $x$ from $h(x)$ to  $h^{\star}(\tilde{x})$, which we use to denote the Bayes optimal prediction for $\tilde{x}$ according to $\ell_{A}$. Recall that in general, the Bayes optimal predictor $h^{\star}$ is maximally accurate (w.r.t $\ell_{A}$), but potentially unfair (w.r.t $\ell_{B}$). Thus, in comparing $h'$ to $h$, accuracy has improved (say by $\tau_{A}(x)$) while fairness has potentially deteriorated (say by $\tau_{B}(x)$). Does $h$ disqualify $h'$? According to Equation (\ref{eqn:intro:disq}), the original classifier $h$ $\gamma$-disqualifies $h'$ if the scaling function $f_{\gamma}$ is such that 

\begin{equation*}
    \tau_{B}(x)>f_{\gamma}(\tau_{A}(x)).
\end{equation*}

Generalizing this argument to arbitrary choices of the initial classifier $h$ and individual $x$, we conclude that w.r.t some fixed scaling function $f_{\gamma}$, the Bayes optimal predictor $h^{\star}$ will be $\gamma^{\star}$-fair, where $\gamma^{\star}\leq\infty$ is the smallest possible value that guarantees that for any initial classifier $h$, the maximal deterioration in fairness $\tau_{B}(x)$ possible from some improvement in accuracy $\tau_{A}(x)$ does not exceed $f_{\gamma^{\star}}(\tau_{A}(x))$. Given a scaling function $f_{\gamma}$, the value of $\gamma^*$ provides a threshold, where improving accuracy (moving individuals to the Bayes optimal prediction) will not lead to disqualification (even though it might degrade fairness).

\vspace{1em}

The above discussion leads to an important principle in choosing the translation function: whenever possible, we want the Bayes optimal predictor to be fair with $\gamma^{\star}=1$. By ``anchoring'' the value of $\gamma^{\star}$ to a fixed value, such as $1$, we guarantee that the extent to which $h^{\star}$ is fair is invariant under scalar multiplication of the loss function $\ell_{A}$. We view this as a desirable property: the accuracy and fairness losses are (initially) incomparable, switching from measuring accuracy using $\ell_{A}$ to measuring accuracy using $100\cdot\ell_{A}$ does not meaningfully change the fairness-accuracy trade-offs in question -- which is precisely what our notion of fairness is aimed at capturing. 
\vspace{1em}

Thus, given a pair of losses $(\ell_{A},\ell_{B})$, we aim to choose the scaling function in a way that guarantees that the Bayes optimal classifier according to $\ell_{A}$ will be $\gamma$-fair with $\gamma=1$. In this section, we fix $\ell_B$ as the 0/1 loss and instantiate this approach w.r.t two common loss functions for measuring accuracy: the squared loss and the 0/1 loss. For the squared loss we put forward a natural scaling function, prove that it satisfies the above requirement, and also show that it has several desirable properties in the regime $\gamma < 1$. On the other hand, for the 0/1 loss, we show that there is no ``reasonable'' scaling function that guarantees the Bayes optimal classifier is $1-$fair. This is because the loss is (in a certain sense) highly ``non-Lipschitz'': achieving tiny accuracy improvements can require an unbounded degradation in fairness. This negative result highlights that, for these loss functions, fairness and accuracy might be wildly conflicting.


\subsection{Scaling w.r.t the squared loss}

For the squared loss  $\ell_A(y, \hat{y}) = (y - \hat{y})^2$ we consider the following scaling function:

\begin{equation}
\label{eqn:scaling-l2}
    f_\gamma(a)=\sqrt{\gamma \cdot \frac{2a}{\eta}}, \quad \text{where} \, \eta=\min_{g\in \set{S,T}}\Pr_{x,y}[x\in g \land y=1]
\end{equation}

The scaling function first applies a linear scaling to account for the fact that fairness and accuracy are computed by aggregating over different-sized sets (akin to the ``warm-up`` discussion above), and then applies a square root to account for the discrepancy between $\ell_A$ and $\ell_B$. We prove that it indeed satisfies the principle described above:

\begin{result}
\label{result:1-fair}
When $\ell_{B}$ is the expected 0-1 loss and $\ell_{A}$ is the squared loss, the Bayes optimal predictor w.r.t $\ell_{A}$, $h^{\star}(x,g)=\E_{\P}[y\vert x,g]$ is $\gamma$-fair for $\gamma=1$ w.r.t the  scaling function defined in Equation (\ref{eqn:scaling-l2}).

\end{result}
See Appendix \ref{appendix:thm1} for the proof. For intuition,
we note that one way in which we can always transform $h^\star$ into a perfectly balanced predictor is to increase the predictions for everyone in $S$ by $ \varepsilon \triangleq Imb(h^\star) $. This transformation increases the squared loss by exactly $\varepsilon^2$ (scaled to the relative mass of $S$). Thus, for this new classifier to not disqualify $h^\star$ when $\gamma=1$, the scaling function should be chosen such that $\varepsilon \leq f_1(\varepsilon^2)$. 

\vspace{1em}
We note that the principle described above does not fully constrain the form of the scaling function, since it only specifies a constraint for $\gamma = 1$. We additionally prove that our choice of the scaling function given in Equation (\ref{eqn:scaling-l2}) (and specifically, the incorporation of $\gamma$ under the square root) guarantees another desirable property:  that convex combinations of the Bayes optimal predictor (which is fair for $\gamma=1$) and the optimal (most accurate) perfectly balanced predictor (which is fair for $\gamma=0$) are also fair (for the class of convex combinations of these two classifiers, and for $\gamma$ specified by the convex combination):

\begin{result}
\label{result:gamma-fairness}
Let $h_1  = \E_{\P}[y\vert x,g]$ and $h_0$ denote the \emph{optimal} $\gamma$-fair classifiers for $\gamma=1$ and $\gamma=0$, respectively. Furthermore, define   $\tilde{\H}$ as the collection of all convex combinations of \emph{some} 1-fair and \emph{some} 0-fair classifiers. Then,
\begin{enumerate}[(i)]
    \item  The classifier $h_\gamma = \gamma \cdot h_1 + (1-\gamma) \cdot h_0 \in \tilde{\H} $ is $(\gamma, \tilde{\H})$-fair, but
    
    \item There is a $0$-fair $h'_0$ such that the classifier  $h'_\gamma = \gamma \cdot h_1 + (1-\gamma) \cdot h'_0 \in \tilde{\H}$  is not $(\gamma, \tilde{\H})$-fair.
\end{enumerate}
\end{result}

See Appendix \ref{appendix:thm2} for the proof. Theorem \ref{result:gamma-fairness} gives some intuition for fairness in the regime $0 < \gamma < 1$, and crucially relies on our choice of scaling function (and in particular on the way that $\gamma$ affects the scaling). The first part of Theorem  \ref{result:gamma-fairness} highlights that in this case, $\gamma \in (0,1)$ serves as a ``knob'' that, as it shrinks, brings down the imbalance level  -- from $dImb(h_1)$ (which is too high when $\gamma < 1$), to a sufficiently low level (here, $\gamma \cdot dImb(h_1)$). However, the second part of the theorem highlights that there is more to our fairness requirement than simply reaching a sufficiently low level of imbalance: Indeed, $h_\gamma$ and $h'_\gamma$ are both \emph{equally imbalanced}, yet only the  first is fair. This is in line with our initial motivation, in which we called for a more relative (context-dependent) perspective, that considers the alternatives and not only the objective level of balance (or imbalance).

\subsection{Scaling w.r.t the 0/1 loss}

 Interestingly, not every combination of losses $\ell_A$ and $\ell_B$ has a scaling function that can guarantee the Bayes optimal classifier is always fair for $\gamma=1$.  We prove this is the case when we switch $\ell_A$ from the squared loss to the expected 0/1 loss.
 
\begin{result}
\label{result:pareto-l1}
Fix $\ell_{B}$ and $\ell_{A}$ to be  the expected 0/1 loss. Then there is no ``reasonable''  scaling function $f_\gamma$ that guarantees that the Bayes optimal classifier for the 0/1 loss is $\gamma$-fair for $\gamma=1$.
\end{result}

Here, the requirement on the scaling function $f_\gamma$ is minimal: we only require that $f_\gamma$ does not ``blow up'' when $\gamma < \infty$. See Appendix \ref{appendix:thm3} for the proof. 
The intuition for Theorem \ref{result:pareto-l1} is that the binary nature of the Bayes optimal classifier for the expected 0-1 loss can work to amplify even very small differences between groups, in a way that our disqualification framework considers unfair (unless $\gamma = \infty$). To see this, consider the simple case in which there is no information except group membership. In this case, the Bayes optimal classifier predicts $0.0$ for an individual from a group whose base rate is below $0.5$, and $1.0$ for an individual from a group whose base rate is above $0.5$. This has the effect of \emph{amplifying} differences between the two groups: even a very small difference in base rates (e.g., $0.49$ for $S$ and $0.51$ for $T$)  translates to a maximal imbalance of 1.0. In this case, our framework views such a classifier as unfair, because e.g. a classifier that predicts  $0.5$ for everyone will disqualify the Bayes optimal classifier: it maximally improves the imbalance (from $1.0$ to $0.0$) at a minimal cost to accuracy.

\vspace{1em}



\section{Fair risk minimization}
\label{section:erm}

While $\H$ itself doesn't necessarily contain a classifier that is $\gamma$-fair w.r.t $\H$, the class of convex combinations of classifiers from $\H$, denoted $\Delta(\H)$, provably does (see Appendix \ref{appendix:thm4} for the proof). Thus, thinking of $\gamma$-fairness as a hard constraint, a natural objective is to output the most accurate classifier in $\Delta(\H)$ that is $\gamma$-fair w.r.t $\H$. We refer to this as the \emph{fair risk minimization} problem. In practice, however, we will typically be interested in the empirical counterpart of this problem (or the \emph{fair ERM}), where all the quantities are computed w.r.t a finite sample. To accommodate the transition from working with the underlying distribution $\P$ to working with an i.i.d. sample $D\sim\P^{m}$, we define a natural notion of \emph{approximate} disqualification, and prove that the resulting approximate fairness definition implied by it generalizes from a sample to the underlying distribution. Intuitively, a classifier $h'$ $(\alpha, \gamma)$-disqualifies $h$ if the requirement in Equation (\ref{eqn:intro:disq}) holds, even when we add an additive slack term $\alpha$ to both the difference in imbalance and the difference in loss. 

\vspace{1em}

In this section we prove that while $\gamma$-fairness and Pareto efficiency are distinct (as solution concepts), the latter is sufficiently informative for the purposes of fair risk minimization. Specifically, we show that given oracle access to an approximation of the Pareto frontier of $\H$ w.r.t accuracy and imbalance (in a \emph{fixed} direction, say $T \to S$) we can solve the fair ERM problem.

\begin{result}
\label{result:fair-erm}
Fix a class $\H$, tradeoff parameter $\gamma$, approximation parameter $\alpha$ and a dataset $D$. There exists an algorithm that, for 
 every $\varepsilon \leq \alpha$,  produces a classifier $h$ that is (i) $(\alpha,\gamma)$-fair w.r.t $\H$ on $D$, (ii) at least as accurate (on $D$) as the optimal $(\alpha-\varepsilon,\gamma)$-fair classifier in $\H$. Furthermore, the algorithm runs in time $\poly(1/\varepsilon)$ and makes $\poly(1/\varepsilon)$ queries to ${\sf{PF}}^{\Delta(\H)}$, where  ${\sf{PF}}^{\Delta(\H)}(\tau)$
 returns a classifier from $\Delta(\H)$ whose accuracy is optimal given that its imbalance level doesn't exceed $\tau \in [-1,1]$.
\end{result}

The algorithm itself is simple: it iterates through the classifiers in the Pareto frontier of $\H$ to find the most accurate one that is not empirically (approximately) disqualified by another classifier on the Pareto frontier. Using the property of Pareto efficiency, this suffices for the desired guarantee; See Appendix \ref{appendix:thm4} for the full proof.

\vspace{1em}

 Theorem \ref{result:fair-erm} demonstrates that the fair ERM problem is no harder than the problem of approximately computing the Pareto frontier of $\H$. In some simple cases, the later  can be solved efficiently. For example, when $\H$ is the class of linear classifiers with bounded norm over $\R^d$, the overall running time of $A$ will be $\poly(1/\varepsilon, d)$ (since every query to ${\sf{PF}}(\tau;\H)$ can be obtained as the solution to a convex program with $d$ variables). While this isn't true for general classes, estimating the entire Pareto frontier is a well studied task with a variety of existing algorithms and heuristics  \cite{fliege2000steepest}.

\emph{Remark.} We note that the algorithm in Theorem \ref{result:fair-erm} only approximately solves the fair ERM problem, in the sense that it produces a $(\alpha,\gamma)$-fair classifier whose sample-accuracy is only competitive with the sample-accuracy of the best $(\alpha-\varepsilon,\gamma)$-fair classifier. Without assuming anything about $\H$, this gap could be substantial. This highlights that the approximation guarantee is closely related to the Lipschitzness of the Pareto frontier of $\H$; see Appendix \ref{appendix:thm4} for an additional discussion.

\section{Discussion}
\label{section:discuss}

\paragraph{Applications.} Our framework can be used to select from (and compare between) different learning strategies, with and without an external quantification of the appropriate trade-off parameter $\gamma$. For a concrete example, suppose a data-scientist fits a standard learning algorithm to their data, say the Adult Income dataset, yielding a predictor $h_{ERM}$ whose squared error on the test set is, say, 0.149. Mindful of fairness, they also evaluate the difference in the true positive rate between men and women, finding it to be significant: under $h_{ERM}$,  men with a positive label receive scores that are on average 0.11 higher than women with positive labels. In response, they consider a different  strategy, such as a adding an explicit fairness regularization penalty to the ERM objective. This yields a second classifier 
$h_{fairERM}$, whose squared error on the test set is now 0.15 but whose imbalance has dropped to 0.051. What should guide the comparison (and eventual choice) between $h_{ERM}$ and $h_{fairERM}$? Our framework provides a simple and clear criteria: we should pick the classifier that improves fairness only if for us, fairness is $20\times$ as important as accuracy is.\footnote{By definition, $h_{fairERM}$ $\gamma$-disqualifies $h_{ERM}$ if the normalized improvement in fairness exceeds the loss to accuracy: $0.11-0.051 > f_\gamma(0.15-0.149)$. The Adult Income dataset is imbalanced, and so in this case $\eta=\Pr[\mathtt{is-woman}\land y=1]\approx 0.03$.  The scaling function $f$ from Theorem 1 therefore ``re-scales'' the loss difference as $f_\gamma(0.001) = \sqrt{2\gamma \cdot 0.001/0.03}$. Solving for $\gamma$, we obtain that the  $h_{fairERM}$ $\gamma$-disqualifies $h_{ERM}$ only when $\gamma < 0.05$.} An important aspect of our framework is that  the value of $\gamma$ can be elicited once from an external party such as a regulator, and this can be applied by data scientists to select classifiers for a variety of different tasks (with the appropriate scaling functions). We also note that the same principles can be used even in the absence of an external quantification of $\gamma$. Fixing a benchmark class $\H$ and a learning strategy that results in a classifier $h$, we can solve for the minimal value  $\hat{\gamma}$ for which there exists $h' \in \H$ that $\hat{\gamma}-$disqualifies $h$. The resulting value $\hat{\gamma}$ can be thought of as the ``effective unfairness'' of the algorithm w.r.t the benchmark, and can serve as additional metric, that is both simple and interpretable, in comparing different learning algorithms.

\paragraph{Disqualification in the ``multi-group`` setup.} While violations of ``group fairness'' notions (such as an imbalance between $S$ and $T$) can serve as meaningful red flags, focusing on aggregate behavior over the entire group can open the door to abuses targeting sub-populations. Individual fairness, pioneered by \cite{dwork2012fairness}, provides strong fairness protections but requires a task-specific similarity metric, which can be challenging to obtain. A more recent line of work, starting with \cite{hebert2017calibration, kearns2018preventing}, strengthens group fairness by requiring its guarantee to hold for every group in a large collections $\cal{C}$ of overlapping sub-populations. Such ``multi-group'' fairness notions can be instantiated with different fairness measures: Multi-calibration \cite{hebert2017calibration} requires that the risk scores be calibrated on every group in the collection whereas Kearns {\em et al.} \cite{kearns2018preventing} suggest imposing demographic parity between groups in the collection. We can thus similarly extend our framework to the multi-group setup by requiring that for 
for every pair of groups $S,T \in {\cal C}$ (where ${\cal C}$ is a large collection of overlapping sets), no classifier $\gamma$-disqualifies the given classifier where the imbalance measures the predictions of $S$ vs $T$. The tension between fairness and accuracy is an important issue in the study of multi-group fairness, as it directly impacts our ability to strengthen the fairness guarantee via enriching the collection $\C$. For fairness notions that are not in tension with accuracy, such as calibration, there is no inherent accuracy downside to enriching ${\cal C}$ (note, however, that guaranteeing fairness for richer collections might require higher running time or sample complexity). In particular, the Bayes-optimal predictor satisfies multi-calibration for any collection of sets \cite{hebert2017calibration}.  On the other hand, for notions like multi-demographic parity, enriching the collection of sets might make informative predictions impossible: e.g. if the collection includes many sets with different base rates. Thus, whereas it would be quite desirable to obtain multi-calibration with respect to entire computation classes (e.g. all the sufficiently-large sets that can be efficiently identified from the data), for multi-demographic parity or multi-balance \cite{kearns2018preventing} this might make very little sense: the collection of sets should be carefully tailored to the problem at hand, and adding additional sets might ``over-constrain'' the problem, forcing the predictions to be uninformative. In this respect, multi-fairness with $\gamma=1$ can be considered as an alternative to the multi-balance notion from \cite{kearns2018preventing} (which is equivalent to multi-fairness with $\gamma=0$): the fairness guarantee is more relaxed, but the collection of sets can be enriched almost arbitrarily without making the definition overly restrictive (a corollary of Theorem \ref{result:1-fair}).

\paragraph{Additional related work.}
The notion of Equal Opportunity was proposed in the seminal work of \cite{hardt2016equality}, and the related notion of balance, which we focus on here, was proposed in \cite{kleinberg2016inherent}. The impossibilities results for simultaneously obtaining calibration (which can be viewed as a minimal accuracy requirement) and error rate parity were established in \cite{kleinberg2016inherent, chouldechova2017fair}. They were strengthened by \cite{pleiss2017fairness} who show that even a single parity constraint (e.g. equalizing false positive rates) can’t be meaningfully obtained together with calibration. The fairness community has taken several approaches in light of these basic trade-offs.  One approach seeks to understand the conditions under which fairness is desirable despite (or in spite) of its cost to accuracy, such as when the data itself is discriminatory \cite{blum2019recovering} or when the fairness dynamics require it \cite{jung2020fair}. Different works attempt to quantify fairness in a way that stands in less stark opposition to accuracy to begin with \cite{hebert2017calibration, kim2019multiaccuracy, blum2019advancing, kallus2019fairness, rothblum2021multi}, or to alleviate the trade-offs by gathering more informative data \cite{chen2018my, garg2019tracking}.

\bibliographystyle{apalike}
\bibliography{refs}

\newpage

\appendix

\section{Additional notation}
\label{appendix:prelims}

We use $\X$ to denote the space of covariates,  $A$ a binary group membership indicator (defining two groups, which we denote by $T$ and $S$), and $\Y = \set{0,1}$. 
We assume an underlying (but unknown) distribution $\P$ on $\X \times A \times \Y$. W will sometimes use $D_g$ to denote $\P$ restricted to samples from group $g \in \set{S,T}$.  Additionally,  $\mu_g$ denotes the fraction of individuals that are in $g$ (with $\mu_S = 1-\mu_T$) and $\beta_g$ denotes the base rate for $g$: $\beta_g \equiv \Pr_{x, y \sim D_g}[y=1]$.

A binary classifier is a mapping $h: \X \times A \to \Y$. A hypothesis class $\H$ is a collection of such binary classifiers. Since we will generally work with convex combinations of binary classifiers, we will consider randomized classifiers  $h: \X \times A \to \hat{\Y}$, where $\hat{\Y} = [0,1]$ and  $h(x, g)$ is the probability that $h$ labels an individual from $g$ with covariates $x$ as positive. 

A loss function $\ell$ is mapping from $Y \times \hat{Y} \to \R$. Slightly abusing notation, we use $\ell_\P(h) = \E_{x,g,y \sim \P}[\ell(h(x,g),y)]$ to denote the loss of a randomized classifier $h$ w.r.t the distribution  $\P$, and we will  drop the subscript $\P$ when it is clear from context. Given a loss $\ell$, we denote by $h^\star_\ell$ the Bayes optimal classifier w.r.t $\ell$: $    h^\star_\ell \in \arg\min_{h: \X \times A \to \hat{\Y}} \ell(h)$.

In this work we will focus on two popular choices: the expected zero-one loss (the probability that $h$ assigns the correct binary label) and the squared loss:

\begin{gather*}
    \ell^{0-1}(\hat{y}, y) = \hat{y} \cdot(1-y)+(1-\hat{y})\cdot y \\
    \ell^2(\hat{y}, y) = (\hat{y}-y)^2
\end{gather*}

The Bayes optimal classifiers for these loss functions are:
\begin{gather*}
 h^\star_{\ell^2}(x, g) = \E_{x,g,y\sim \P}\sbr{y \vert x, g},  \\ 
h^\star_{\ell^{0-1}}(x,g) = \textbf{1} \sbr{\E_{x,g,y\sim \P}\sbr{y \vert x, g} \geq \tfrac{1}{2}} = \textbf{1} \sbr{h^\star_{\ell^2}(x, g) \geq \tfrac{1}{2}}
\end{gather*}

\section{Loss imbalance generalizes existing notions}
\label{appendix:loss-imbalance}

Our starting point is a definition of \emph{group fairness} that seeks to equalize some quantity $\psi(h,C)$ across groups $S$ and $T$. Assuming higher values of $\psi(h,C)$ are better, we can then define the (directed) degree of unfairness as

\begin{equation}
\label{eqn:general-unfairness}
    \text{dUnfairness}(h; \psi, C\to C') = \psi(h,C) - \psi(h,C')
\end{equation}

Such that $\text{dUnfairness}(h; \psi, T \to S)$ measures the extent to which $T$ receives favourable outcomes (in the sense implied by $\psi$). For example,

\begin{itemize}
    \item $\psi^{DP}(h,C) = \E_{x,g,y \sim \P}[h(x) \, \vert\, g=C]$ recovers the definition of Demographic Parity, 
    \item  $\psi^{PosBalance}(h,g) = \E_{x,g,y \sim \P}[h(x)\,  \vert \, y=1, g=C]$ recovers the definition of balance for the positive class \cite{kleinberg2016inherent}, and, when $h$ is binary, also the definition of Equal Opportunity \cite{hardt2016equality} 
\end{itemize}

\emph{Remark.} While this is not the focus of our work, we can generalize this formulation to situations in which by fairness we mean equalizing multiple quantities. For example, in Equal Odds \cite{hardt2016equality}, both $\psi_1(h,g)=\FPR(h,g)$ and $\psi_2(h,g)=\TPR(h,g)$ should be equalized across groups; in this case, the directed unfairness can be defined more generally as
\begin{equation}
    \text{dUnfairness}(h; \psi, C\to C') = \max\set{\psi_1(h,C') - \psi_1(h,C), \, \, \psi_2(h,C) - \psi_2(h,C')}
\end{equation}

\subsection{Imbalance and loss imbalance}
\label{sec:loss_imb}

In this work, our focus will be on taking fairness to mean balance for the positive class. We refer to the directed unfairness  $\text{dUnfairness}(h; \psi^{PosBalance}, T \to S)$ more simply as the directed imbalance, denoted $dImb(h; T\to S)$. In fact, we will be working with a slightly more  general notion of imbalance, which we refer to as loss imbalance. Instead of directly comparing the expected scores of the positive members of $T$ with expected scores of the positives members of $S$,  loss imbalance compares the difference in their expected \emph{losses}. 
 
 \begin{definition}[Loss Imbalance]
 \label{def:loss_imb}
 Given a loss $\ell: \hat{\Y} \times \Y \to \R$, the loss imbalance of $h$ in the direction $T \to S$ is defined as
 \begin{equation}
\label{eqn:imbalance}
    \text{dLossImb}(h; \, T \to S, \ell) = \psi^{PosLossBalance}(h,S) - \psi^{PosLossBalance}(h,T)
\end{equation}

where $\psi^{PosLossBalance}(h,C) = \E_{x,  g, y \sim \P}[\ell(h(x), y)  \, \vert \, y=1, g = C]$.
 \end{definition}
 
 Note the change in order -- subtracting $T$ from $S$, as opposed to $S$ from $T$ in Equation (\ref{eqn:general-unfairness}). This is because for loss $\ell$, lower is better.

Loss imbalance is a generalization of the notion of imbalance. The next lemma shows that imbalance is simply the loss imbalance as measured w.r.t the expected zero one loss.

\begin{lemma}
\label{lemma:0-1-loss-imb}
For any $h: \X \to \hat{\Y}$, $dLossImb(h; \ell^{0-1}) = dImb(h)$.
\end{lemma}

\begin{proof}
Recall that the 0-1 loss is defined as $\ell^{0-1}(h(x),y)=h(x)\cdot(1-y)+(1-h(x))\cdot y$, so $\ell^{0-1}(h(x),1)=1-h(x)$. We therefore have

\begin{align*}
    LossImb(h;\ell^{0-1})	&= \E_{x,  g, y \sim \P}[\ell(h(x), y)  \, \vert \, y=1, g = S] - \E_{x,  g, y \sim \P}[\ell(h(x), y) \, \vert \, y=1, g = T] \\
    &= \E_{x,  g, y \sim \P}[\ell(h(x), 1)  \, \vert \, y=1, g = S] - \E_{x,  g, y \sim \P}[\ell(h(x), 1) \, \vert \, y=1,  g = T] \\
	&=\E_{x,  g, y \sim \P}[1 - h(x) \, \vert \,y=1, g = S] - \E_{x,  g, y \sim \P}[1- h(x)\, \vert \,  y=1, g = T] \\
	&= \E_{x,  g, y \sim \P}[h(x)  \, \vert \, y=1, g = T] - \E_{x,  g, y \sim \P}[h(x) \, \vert \, y=1, g = S] \\
	&=Imb(h)
\end{align*}

\end{proof}

\section{Proof of Theorem \ref{result:1-fair}}
\label{appendix:thm1}

\subsection{Warmup: A simple case}

For simplicity, let $\ell$ denote the squared loss. We want to show a choice of $f=f_1$ for which no classifier can disqualify $\ensuremath{h_{\ell^{2}}^{\star}}$:

\begin{equation*}
    \forall h': \X \to [0,1]: \quad Imb(\ensuremath{h_{\ell^{2}}^{\star}}) - Imb(h') \leq f(\min \set{0, \,\, \ell(h') - \ell(\ensuremath{h_{\ell^{2}}^{\star}})})
\end{equation*}

Since $\ensuremath{h_{\ell^{2}}^{\star}}$ is optimal, this is  the same as 

\begin{equation*}
    \forall h': \X \to [0,1]: \quad Imb(\ensuremath{h_{\ell^{2}}^{\star}}) - Imb(h') \leq f( \ell(h') - \ell(\ensuremath{h_{\ell^{2}}^{\star}}))
\end{equation*}

We will show this can be done in the simple case 
in which $\X = \emptyset$, and $\ensuremath{h_{\ell^{2}}^{\star}}$ predicts $\beta_T$ for $T$ and $\beta_S$ for $S$, and has an imbalance of $\beta_T - \beta_S$. There are two ways for a classifier to improve the imbalance: increase the prediction for $S$ or decrease the prediction for $T$. Consider the first case, so $h'$ is identical to $\ensuremath{h_{\ell^{2}}^{\star}}$, 
 but  adds $\varepsilon$ to the prediction of $S$. By definition, the imbalance improves by $\varepsilon$. The increase in loss, on the other-hand,  is proportional to $\varepsilon^2$:

\begin{align*}
    \ell(h')-\ell(\ensuremath{h_{\ell^{2}}^{\star}})	&=d(h',\ensuremath{h_{\ell^{2}}^{\star}})-d(\ensuremath{h_{\ell^{2}}^{\star}},\ensuremath{h_{\ell^{2}}^{\star}}) \\
	&=d(h',\ensuremath{h_{\ell^{2}}^{\star}}) \\
	&=\mu_{S}\cdot\varepsilon^{2}
\end{align*}

Where $d(p,p') = \E_{x}[(p(x)-p'(x))^2]$. This uses the fact that for every classifier $h$, the squared loss is related to the Euclidean distance from $\ensuremath{h_{\ell^{2}}^{\star}}$, as follows: $\ell(h) = d(h,\ensuremath{h_{\ell^{2}}^{\star}}) + t(\ensuremath{h_{\ell^{2}}^{\star}})$; Indeed:

\begin{align*}
    \ell(h)&=\E_{x,y}[(y-h(x))^{2}] \\
    &= \E_{x}[\ensuremath{h_{\ell^{2}}^{\star}}(x)\cdot(1-h(x))^{2}+(1-\ensuremath{h_{\ell^{2}}^{\star}}(x))\cdot h(x)^{2}] \\
    &= \E_{x}[\ensuremath{h_{\ell^{2}}^{\star}}(x)-2\ensuremath{h_{\ell^{2}}^{\star}}(x)h(x)+h(x)^{2}] \\
    &= \E_{x}\ensuremath{h_{\ell^{2}}^{\star}}(x)-\ensuremath{h_{\ell^{2}}^{\star}}(x)^{2}]+\E_{x}[(\ensuremath{h_{\ell^{2}}^{\star}}(x)-h(x))^{2}] \\
    &=d(h,\ensuremath{h_{\ell^{2}}^{\star}}) +  t(\ensuremath{h_{\ell^{2}}^{\star}})
\end{align*}

We can now see that $h'$ doesn't $1-$disqualify $\ensuremath{h_{\ell^{2}}^{\star}}$:


\begin{equation*}
    f_1(\ell(h')-\ell( \ensuremath{h_{\ell^{2}}^{\star}})) = t_1(1) \cdot t_2(\mu_{S}\cdot\varepsilon^{2}) = t_2(\mu_{S}\cdot\varepsilon^{2}) = \sqrt{\frac{\mu_S \cdot \varepsilon^2}{\mu \cdot \beta}} \geq \varepsilon = Imb( \ensuremath{h_{\ell^{2}}^{\star}})-Imb(h')
\end{equation*}

as required.

\subsection{Full proof of Theorem \ref{result:1-fair}}

\begin{proof}
For simplicity, denote $h^\star \triangleq \ensuremath{h_{\ell^{2}}^{\star}}$. Assume $h^\star$ is not perfectly balanced (otherwise we are done) and w.l.o.g that the imbalance is in favor of $T$, so when we write Imb we mean the imabalance in the direction $T\to S$. 

We want to prove that for the specified scaling function, no other classifier $h$ $1-$disqualifies $h^\star$:

\begin{equation}
\label{eqn: h_star_is_fair}
     f_1(\ell(h) - \ell(h^\star)) \geq \text{dImb}(h^\star) - \text{dImb}(h)
\end{equation}

For simplicity, we will sometimes write $f_1 = f$. Note that we can  indeed write $f(\ell(h) - \ell(h^\star))$ (as opposed to $f(\max\set{0,\ell(h) - \ell(h^\star)})$, as the definition states) because from the optimality of $h^\star$ w.r.t $\ell=\ell^2$ we have that $\ell(h) - \ell(h^\star) \geq 0$.

Let $h$ be any classifier; define the following quantities for $x\in \X, g\in \set{S,T}$:

\begin{align*}
    \Delta_{x,g} &\equiv h^\star(x,g) - h(x,g) \\
    m_{x,g} &\equiv \Pr_{D_g}(X=x) \\
    m^y_{x,g} &\equiv \Pr_{D_g}(X=x \vert Y=y) 
\end{align*}

Note that for a group $g$ and $y \in \set{0,1}$, both define legal probability measures (they are non-negative, and sum to 1: $\sum_{x\in \X}m_{x,g}= \sum_{x \in \X}m_{x,g}^y=1$). We can now express both the difference in loss and the difference in imbalance between $h^\star$ and $h$ in terms of $m, m^1$ and $\Delta$, as follows:

\begin{gather}
    Imb(h^{\star})-Imb(h) = \sum_{x}m_{x,T}^{1}\cdot\Delta(x,T)-\sum_{x}m_{x,S}^{1}\cdot\Delta(x,S) \label{eqn:imb-diff} \\
    \ell(h)-\ell(h^{\star})=\mu_{T}\cdot\sum_{x}m_{x,T}\cdot\Delta(x,T)^{2}+\mu_{S}\cdot\sum_{x}m_{x,S}\cdot\Delta(x,S)^{2} \label{eqn:loss-diff}
\end{gather}

For (\ref{eqn:imb-diff}), we first note that for any classifier $\tilde{h}$, 
\begin{equation*}
    Imb(\tilde{h})	=\E[\tilde{h}(x,T)\vert y=1]-\E[\tilde{h}(x,S)\vert y=1]
	=\sum_{x}m_{x,T}^{1}\cdot h(x,T)-\sum_{x}m_{x,S}^{1}\cdot h(x,S)
\end{equation*}

So 
\begin{align*}
    Imb(h^{\star})-Imb(h)	&=\left(\sum_{x}m_{x,T}^{1}\cdot h_{1}(x,T)-\sum_{x}m_{x,S}^{1}\cdot h_{1}(x,S)\right)-\left(\sum_{x}m_{x,T}^{1}\cdot h(x,T)-\sum_{x}m_{x,S}^{1}\cdot h(x,S)\right) \\
	&=\sum_{x}m_{x,T}^{1}\cdot\Delta(x,T)-\sum_{x}m_{x,S}^{1}\cdot\Delta(x,S)
\end{align*}

For (\ref{eqn:loss-diff}), note that the loss of a classifier $\tilde{h}$ restricted to a group $g$ can be written as

\begin{align*}
    \ell_{g}(\tilde{h})	&=\sum_{x}m_{x,g}\cdot\left(h^{\star}(x,g)\cdot(\tilde{h}(x,g)-1)^{2}+(1-h^{\star}(x,g)\cdot\tilde{h}(x,g)^{2}\right) \\
	&=\sum_{x}m_{x,g}\cdot\left(h^{\star}(x,g)-2h^{\star}(x,g)\tilde{h}(x,g)+\tilde{h}(x,g)^{2}\right)
\end{align*}

By direct calculation, this implies that 
\begin{equation*}
    \ell_{g}(h)-\ell_{g}(h^{\star})=\sum_{x}m_{x,g}\cdot(h(x,g)-h^{\star}(x,g))^{2}=\sum_{x}m_{x,g}\cdot\Delta(x,g)
\end{equation*}

From which (\ref{eqn:loss-diff}) follows, since $\ell(h)-\ell(h^{\star})=\mu_{T}\cdot\left[\ell_{T}(h)-\ell_{T}(h^{\star})\right]+\mu_{S}\cdot\left[\ell_{S}(h)-\ell_{S}(h^{\star})\right]$.

We can now obtain the required:

\begin{align}
    f_{1}\left[\ell(h)-\ell(h^\star)\right]	&=\sqrt{\frac{2\left[\mu_{T}\cdot\sum_{x}m_{x,T}\cdot\Delta(x,T)^{2}+\mu_{S}\cdot\sum_{x}m_{x,S}\cdot\Delta(x,S)^{2}\right]}{\mu\beta}} \label{thm-h-star-t1}\\
	&\geq\frac{\sqrt{\mu_{T}\cdot\sum_{x}m_{x,T}\cdot\Delta(x,T)^{2}}+\sqrt{\mu_{S}\cdot\sum_{x}m_{x,S}\cdot\Delta(x,S)^{2}}}{\sqrt{\mu\beta} } \label{thm-h-star-t2} \\
	&\geq\frac{\sqrt{\sum_{x}m_{x,T}\cdot\Delta(x,T)^{2}}+\sqrt{\sum_{x}m_{x,S}\cdot\Delta(x,S)^{2}}}{\sqrt{\beta}} \label{thm-h-star-t3}
 \\	&\geq\frac{\sqrt{\beta_{T}\cdot\sum_{x}m_{x,T}^{1}\cdot\Delta(x,T)^{2}}+\sqrt{\beta_{S}\sum_{x}m_{x,S}^{1}\cdot\Delta(x,S)^{2}}}{\sqrt{\beta}}  \label{thm-h-star-t4} \\
	&\geq\sqrt{\sum_{x}m_{x,T}^{1}\cdot\Delta(x,T)^{2}}+\sqrt{\sum_{x}m_{x,S}^{1}\cdot\Delta(x,S)^{2}} \label{thm-h-star-t5} \\
	&\geq\sum_{x}m_{x,T}^{1}\cdot\Delta(x,T)-\sum_{x}m_{x,S}^{1}\cdot\Delta(x,S) \label{thm-h-star-t6} \\
	&=Imb(h^{\star})-Imb(h) \label{thm-h-star-t7}
\end{align}

where:  (\ref{thm-h-star-t1}) follows directly by applying $f_1$ to the expression we derived in Equation (\ref{eqn:loss-diff}) for the loss difference; the transition in (\ref{thm-h-star-t2}) follows from the fact that $\sqrt{2(a+b)} \geq \sqrt{a} + \sqrt{b}$ for $a,b \geq 0$; the transition in (\ref{thm-h-star-t3}) follows by the definition of $\mu$ as $\min\set{\mu_T, \mu_S}$; the transition in (\ref{thm-h-star-t4}) uses the fact that $m_{x,g} \geq m_{x,g}^{1}\cdot\beta_{g}$, which follows from the law of total probability:

\begin{align*}
    m_{x,g} &= \Pr_{D_g}(X=x) \\
    &= \sum_{y}\sbr{\Pr_{D_g}\br{X=x, Y=y}} \\
    &= \sum_{y}\sbr{\Pr_{D_g}\br{X=x \vert Y=y} \cdot \Pr_{D_g}(Y=y)} \\
    &\geq \Pr_{D_g}\br{X=x \vert Y=1} \cdot \Pr_{D_g}(Y=1) \\
    &= m^1_{x,g} \cdot \beta_g
\end{align*}

The transition in (\ref{thm-h-star-t5}) follows by the definition of $\beta$ as $\min\set{\beta_T, \beta_S}$, and the transition in (\ref{thm-h-star-t6}) is an application of Jensen's inequality (in the finite sum version), which states that $\varphi\left(\frac{\sum a_{i}x_{i}}{\sum a_{i}}\right)\geq\frac{\sum a_{i}\varphi(x_{i})}{\sum a_{i}}$ for concave $\varphi$. Indeed, using the fact that $\sum_{x}m_{x,g}^{1} = 1$,

\begin{equation*}
    \sqrt{\sum_{x}m_{x,g}^{1}\cdot\Delta(x,g)^{2}}=\sqrt{\frac{\sum_{x}m_{x,g}^{1}\cdot\Delta(x,g)^{2}}{\sum_{x}m_{x,g}^{1}}}\geq\frac{\sum_{x}m_{x,g}^{1}\cdot\sqrt{\Delta(x,g)^{2}}}{\sum_{x}m_{x,g}^{1}}=\frac{\sum_{x}m_{x,g}^{1}\cdot\left|\Delta(x,g)\right|}{\sum_{x}m_{x,g}^{1}} = \sum_{x}m_{x,g}^{1}\cdot\left|\Delta(x,g)\right|
\end{equation*}

Finally, the transition in (\ref{thm-h-star-t7}) uses the expression we derived for the imbalance difference in Equation (\ref{eqn:imb-diff}).

\end{proof}

\section{Proof of Theorem \ref{result:gamma-fairness}}
\label{appendix:thm2}

The proof of the first part of Theorem \ref{result:gamma-fairness} will be based on the following lemma, which we state and prove first.

\begin{lemma}
\label{lemma:h_a_h_b}
Fix any $0 \leq a \leq b \leq 1$. If $\gamma \geq b$, then $h_a = a \cdot h_1 + (1-a) \cdot h_0$ doesn't $\gamma-$disqualify $h_b = b \cdot h_1 + (1-b) \cdot h_0$.
\end{lemma}

\begin{proof}

We want to prove that $h_a$ doesn't $\gamma-$disqualify $h_b$:

\begin{equation*}
    Imb(h_{b})-Imb(h_{a})\leq f_{\gamma}\left(\min\left\{ 0,\ell(h_{a})-\ell(h_{b})\right\} \right)
\end{equation*}

By definition, $h_a$ is less accurate and more balanced than $h_b$,  so we can re-write the above as  

\begin{equation}
\label{eqn:h_gamma_lemma_required}
    Imb(h_{b})-Imb(h_{a})\leq f_{\gamma} \br{\ell(h_{a})-\ell(h_{b})}
\end{equation}

Let's start with the left-hand-side. Since imbalance is a linear operator, we have that in general, $Imb(h_\alpha) = Imb(\alpha \cdot h_1 + (1-\alpha)\cdot h_0) = \alpha \cdot Imb(h_1) + (1-\alpha) \cdot Imb(h_0) = \alpha \cdot Imb(h_1)$, where the last transition follows from the fact that $h_0$ is, by definition, perfectly balanced. We therefore have that

\begin{equation*}
    Imb(h_{b})-Imb(h_{a}) = (b-a) \cdot Imb(h_1)
\end{equation*}

Next, recall that $h_1$ is $1-$fair, so $h_0$ doesn't $1-$disqualify it:

\begin{equation*}
    Imb(h_{1})-Imb(h_{0})\leq f_1\left(\min\left\{ 0,\ell(h_{0})-\ell(h_{1})\right\} \right)
\end{equation*}

Equivalently, $Imb(h_{1})\leq f_1\left(\ell(h_{0})-\ell(h_{1})\right)$. Putting it together, for the inequality in (\ref{eqn:h_gamma_lemma_required}) to hold, it suffices that

\begin{equation}
\label{eqn:h_gamma_lemma_int}
    (b-a) \cdot f_{1} \br{\ell(h_{0})-\ell(h_{1})} \leq f_{\gamma} \br{\ell(h_{a})-\ell(h_{b})}
\end{equation}

Denote $c = \ell(h_{0})-\ell(h_{1})$. We will be using the following fact:
\begin{equation}
\label{eqn:h_gamma_lemma_aux_fact}
     \ell(h_{a})-\ell(h_{b}) = (b-a)\cdot (2-a-b)\cdot c
\end{equation}

We will first show how, assuming Equation (\ref{eqn:h_gamma_lemma_aux_fact}) is true, we can complete the proof of the lemma.  Using the definition of $f$ as $f_\gamma(a) = \sqrt{\frac{2\gamma a}{\mu \beta}}$,  Equation (\ref{eqn:h_gamma_lemma_int}) is equivalent to

\begin{equation*}
    (b-a) \cdot  \sqrt{\frac{2\cdot c}{\mu \beta}} \leq \sqrt{\frac{2\cdot c \cdot \gamma \cdot (b-a)(2-a-b)}{\mu \beta}}
\end{equation*}

Or $(b-a) \leq  \gamma \cdot (2-a-b)$. It's therefore left to show that $\gamma \geq b$ implies

\begin{equation}
    \gamma \geq \frac{b-a}{2-a-b}
\end{equation}

Denote $b=1-\tau$ and $a=1-\tau-\varepsilon$; where the assumption $0\leq a\leq b\leq1$ implies $\tau \in [0,1]$ and $\varepsilon \in [0, 1-\tau]$. Indeed:

\begin{align*}
    \frac{b-a}{2-a-b}	&=\frac{\varepsilon}{2\tau+\varepsilon} \\
	&=\frac{1}{1+2\tau/\varepsilon} \\
	&\leq\frac{1}{1+2\tau/(1-\tau)} \\
	&=\frac{1-\tau}{1+\tau} \\
	&\leq1-\tau \\ 
	&=b \\
	&\leq\gamma
\end{align*}

To conclude the lemma, we return to the proof of the fact in Equation (\ref{eqn:h_gamma_lemma_aux_fact}):

\begin{equation*}
     \ell(h_{a})-\ell(h_{b}) = (b-a)\cdot (2-a-b)\cdot \ell(h_{0})-\ell(h_{1})
\end{equation*}

Indeed, note that for a classifier $h$, we can write the squared loss as 

\begin{align*}
    \ell(h)&= \E_{x,y}[\ell(h(x),y)] \\
    &= \E_{x}\left[h_{1}(x)\cdot\left(1-h(x)\right)^{2}+(1-h_{1}(x))\cdot h(x)^{2}\right] \\
	&=\E_{x}\left[h_{1}(x)-2h(x)h_{1}(x)+h(x)^{2}\right]
\end{align*}

This means that the difference $\ell(h_a) - \ell(h_b)$ equals

\begin{align}
\label{eqn:l_a-l_b}
    \ell(h_{a})-\ell(h_{b}) &= \E_{x}\left[-2h_a(x)h_{1}(x)+h_a(x)^{2} +2h_b(x)h_{1}(x)-h_b(x)^{2}\right] \\ &= \E_{x}\left[(h_{a}(x)-h_{b}(x))\cdot(h_{a}(x)+h_{b}(x)-2h_{1}(x))\right]
\end{align}

Additionally, we can write $h_a - h_b$ and $h_a + h_b$ in terms of $h_1$ and $h_0$, as follows:

\begin{gather*}
    h_a - h_b = (a-b)\cdot(h_{1}-h_{0}) \\
    h_a + h_b = (a+b)\cdot h_{1}-(a+b-2)\cdot h_{0}
\end{gather*}
So $h_a + h_b -2h_1 = (a+b-2)\cdot(h_{1}-h_{0})$, and plugging this back into the expression in (\ref{eqn:l_a-l_b}), we get:

\begin{equation*}
    \ell(h_a) - \ell(h_b) = (a-b)\cdot(a+b-2)\cdot \E_{x}\left[(h_{1}(x)-h_{0}(x))^{2}\right]
\end{equation*}

In particular, for $a=0, b=1$ we get:

\begin{equation*}
     \ell(h_0) - \ell(h_1) =  \E_{x}\left[(h_{1}(x)-h_{0}(x))^{2}\right]
\end{equation*}

which is exactly what we wanted to show.

\end{proof}

We will now use the lemma to prove Theorem \ref{result:gamma-fairness}.

\begin{proof}
\emph{(i).} Let $h' = \alpha \cdot h_1 + (1-\alpha)\cdot h'_0$ be some classifier in $\tilde{\H}$, where $h'_0$  is some 0-fair classifier and $0\leq \alpha \leq 1$.  We need to prove that $h'$ does not $\gamma-$disqualify $h_\gamma$.  To do so, we argue that (a) $h_\alpha$ doesn't $\gamma$-disqualify $h_\gamma$ and (b) if $h_\alpha$ doesn't $\gamma$-disqualify $h_\gamma$, then neither does $h'$. 

For (a), note that this is an immediate corollary from Lemma \ref{lemma:h_a_h_b}.
First, for  $h'$ to disqualify $h_\gamma$ it must be more balanced, so $\alpha < \gamma$. But then the lemma guarantees that $h_\alpha$ doesn't disqualify $h_\gamma$ (with $h_\alpha$ in the role of $h_a$ and $h_\gamma$ in the role of $h_b$, so clearly the assumption of the lemma is true because $\gamma \geq b$ is the same as $\gamma \geq \gamma$).

For (b), we will be using the fact that $Imb(h') = Imb(h_\alpha)$ (which follows directly from the definition), our claim in (a), and the fact that $\ell(h_\alpha) \leq \ell(h')$ (which we will prove promptly), in order. This yields:

\begin{equation*}
    Imb(h_\gamma) - Imb(h') = Imb(h_\gamma) - Imb(h_\alpha) \leq f_\gamma(\ell(h_\alpha) - \ell(h_\gamma)) \leq f_\gamma(\ell(h') - \ell(h_\gamma))
\end{equation*}

which means that $h'$ doesn't $\gamma$-disqualify $h_\gamma$, which is what we wanted to show.

It is left to prove the fact that $\ell(h_\alpha) \leq \ell(h')$, or $\ell(\alpha \cdot h_1 + (1-\alpha) \cdot h'_0) \leq \ell(\alpha \cdot h_1 + (1-\alpha) \cdot h_0)$.

First, note that $h'-h_{\alpha} = (1-\alpha)\cdot(h'_{0}-h_{0})$ and $h'+h_{\alpha} = 2\alpha\cdot h_{1}+(1-\alpha)\cdot(h'_{0}+h_{0})$. The latter means that $h'+h_{\alpha}-2h_{1} = (1-\alpha)\cdot(h'_{0}+h_{0}-2h_{1})$. Now,

\begin{align*}
    \ell(h')-\ell(h{}_{\alpha})	&=\E[h_{1}-2h'h_{1}+h'^{2}-h_{1}+2h_{\alpha}h_{1}-h_{\alpha}^{2}] \\
	&=\E[-2h'_{\alpha}h_{1}+h'^{2}+2h_{\alpha}h_{1}-h_{\alpha}^{2}] \\
	&=\E[2h_{1}(h_{\alpha}-h')+(h'-h_{\alpha})(h'+h)] \\
	&=\E[(h'-h_{\alpha})(h'+h_{\alpha}-2h_{1})] \\
	&= \E[(1-\alpha)\cdot(h'_{0}-h_{0})(1-\alpha)\cdot(h'_{0}+h_{0}-2h_{1})] \\
	&= (1-\alpha)^{2}\cdot\E[(h'_{0}-h_{0})(h'_{0}+h_{0}-2h_{1})] \\
	&= (1-\alpha)^{2}\cdot\E[h'_{0}(x)^{2}-2h_{1}h'_{0}-h_{0}(x)^{2}+2h_{1}h_{0}] \\&= (1-\alpha)^{2}\cdot\E[h'_{0}(x)^{2}-2h_{1}h'_{0}-h_{0}(x)^{2}+2h_{1}h_{0} + h_1^2 - h_1^2] \\
	&=(1-\alpha)^{2}\cdot\left[\E[h'_{0}(x)^{2}-2h_{1}h'_{0}+h_{1}(x)^{2}]-\E[h{}_{0}(x)^{2}-2h_{1}h{}_{0}+h_{1}(x)^{2}\right]\\
	&=(1-\alpha)^{2}\cdot\left[d(h'_{0},h_{1})-d(h_{0},h_{1})\right]\\
	&=(1-\alpha)^{2}\cdot\left[\ell(h'_{0})-\ell(h_{0})\right] \\
	&\geq 0
\end{align*}

The final transition uses the fact that $\alpha \leq 1$ and $\left[\ell(h'_{0})-\ell(h_{0})\right] \geq 0$ (which is true because  $h_0$ was defined to be the most accurate 0-fair predictor).


\vspace{1em}

\emph{(ii).} We turn to prove the second part of Theorem \ref{result:gamma-fairness}. Recall
we want to show that there exists a $0$-fair classifier $h'_0 \neq h_0$, such that  $h'_\gamma = \gamma \cdot h_1 + (1-\gamma) \cdot h'_0 \in \tilde{\H}$  is \emph{not} $(\gamma, \tilde{\H})$-fair.

We will construct such an instance, as follows. Let $S^+, S^-, T^+, T^-$ denote the positive and negative subsets of $S$ and $T$, respectively (Note that these groups may not be explicitly defined in $\X$). Suppose that the features $\X$ are such that only the following subsets can be identified from $\X$: $S_1$ (consisting of all of $S^+$ and half of $S^-$), $T_1$ (consisting of all of $T^+$ and half of $T^-$) and $T^+$. This means the optimal classifier is

\begin{equation*}
    h_1(x) = 
    \begin{cases}
    1, & x\in T^+ \\
    2/3, & x\in S_1 \\
    0, &\text{otherwise}
    \end{cases}
\end{equation*}

Note that $h_1$ has an imbalance of $1/3$. Consider two 0-fair classifiers: 

\begin{equation*}
    g_0(x) = \begin{cases}
    2/3, & x\in T_1 \\
    2/3, & x\in S_1 \\
    0, &\text{otherwise}
    \end{cases}, \qquad g'_0(x) = 0.5
\end{equation*}

And consider $h'_\gamma = \gamma\cdot h_1 + (1-\gamma) \cdot g'_0$ vs $h_\gamma = \gamma\cdot h_1 + (1-\gamma) \cdot g_0$. We want to use the second to $\gamma$-disqualify the first. Towards this, consider the classifier $h_{\gamma-\varepsilon}$ -- a mix of $h_1$ and $g_0$ that puts slightly less weight on $h_1$. Relative to $h'_\gamma$, this classifier \emph{improves} the imbalance by $\varepsilon$. However, we argue that it is also more accurate (for a sufficiently small $\varepsilon$); indeed, as we take $\varepsilon$ to zero, the difference $\ell(h_{\gamma - \varepsilon} - \ell(h'_\gamma)$ tends to the difference $\ell(h_\gamma) - \ell(h'_\gamma) $, and the latter is strictly positive for every $\gamma$ (since $g_0$ is more accurate than $g'_0$).   We therefore have that for any level of $\gamma$, for a sufficiently small $\varepsilon$, $h_{\gamma - \varepsilon}$ is both more accurate and more balanced than $h'_\gamma$; thus, by definition, $h_{\gamma - \varepsilon}$ $\gamma$-disqualifies $h'_\gamma$. Since $h_{\gamma - \varepsilon} \in \tilde{\H}$, we conclude that $h'_\gamma$ is not ($\gamma, \tilde{\H})$-fair, as required.

\end{proof}

\section{Proof of Theorem \ref{result:pareto-l1}}
\label{appendix:thm3}

Before we turn to proving Theorem 3, we formalize the notion of ``reasonable'' scaling function.

\begin{definition}
A scaling function $f: \R^+ \times [0,1] \to \R^+$ is legal if it can be written as
\begin{equation*}
    f(\gamma, a) = t_1(\gamma) \cdot t_2(a)
\end{equation*}
such that the following requirements hold: (i) $t_1$ and $t_2$ are non-decreasing, (ii)  $t_1(\cdot)$ is zero at $\gamma =0$ and tends to infinity as $\gamma \to \infty$ (and only as $\gamma \to \infty$),  and (iii) $t_2(0) =0$.

We will often write the scaling function as $f_\gamma(a)$ (i.e., as a function from $[0,1]$  to $\R^+$, parameterized by $\gamma \in \R^+$). Finally, we use $\F$ to denote the set of all such legal scaling functions.
\end{definition}

To prove Theorem \ref{result:pareto-l1}, we again consider the simple example in which $\X = \emptyset$ (so the only information available for prediction is group membership). In this case, the Bayes optimal classifiers for the squared loss and expected zero one loss, respectively, are:

\begin{equation*}
    h^\star_{\ell^2} (x, g) = \beta_g, \qquad   h^\star_{\ell^{0-1}} (x, g) = \textbf{1} [\beta_g > 0.5]
\end{equation*}


Suppose that for some positive $\tau$, $\beta_T = 0.5 + \tau$ and $\beta_S = 0.5  - \tau$. Note that in this case $h^\star_{\ell^{0-1}}$ is maximally imbalanced, since it predicts $1.0$ for all members of $T$ and $0.0$ for all members of $S$.

Now, consider $h^\star_{\ell^2}$ as the alternative classifier. When does $h^\star_{\ell^2}$ 1-disqualify $h^\star_{\ell^{0-1}}$? In terms of imbalance, its imbalance is $\beta_T - \beta_S = 2\tau $, so the improvement in imbalance is $1.0 - 2\tau$. In terms of accuracy, we have:

\begin{gather*}
    \ell(h_{\ell^{0-1}}^{\star})=\mu_{S}\beta_{S}+\mu_{T}(1-\beta_{T}) \\
    \ell(\ensuremath{h_{\ell^{2}}^{\star}})=2\cdot\left[\mu_{S}\beta_{S}\cdot(1-\beta_{S})+\mu_{T}\beta_{T}\cdot(1-\beta_{T})\right]
\end{gather*}

So the difference in loss is

\begin{align*}
   \ell(\ensuremath{h_{\ell^{2}}^{\star}})  - \ell(h_{\ell^{0-1}}^{\star}) &= \mu_{S}\beta_{S}(2-2\beta_{S}-1)+\mu_{T}(1-\beta_{T})(2\beta_{T}-1) \\
   &= 2\tau\cdot\left[\mu_{S}\beta_{S}+\mu_{T}(1-\beta_{T})\right]
\end{align*}

In order for $h_{\ell^{0-1}}^{\star}$ to be $\gamma-$fair, the following must hold:
\begin{equation}
\label{eqn:l1-l1-unfair}
    1-2\tau\leq f_{\gamma}(2\tau\cdot\left[\mu_{S}\beta_{S}+\mu_{T}(1-\beta_{T})\right]) = t_1(\gamma) \cdot t_2(2\tau\cdot\left[\mu_{S}\beta_{S}+\mu_{T}(1-\beta_{T})\right])
\end{equation}

Note that if $f$ is a legal scaling function, $t_1(\gamma) $ is bounded for $\gamma < \infty$, e.g. by $M$. This means that as we take $\tau \to 0$, the RHS of Equation (\ref{eqn:l1-l1-unfair}) approaches $M \cdot 0 = 0$, whereas the LHS approaches 1.0. 
 This shows that there cannot be a single $f \in \F$ that guarantees $h_{\ell^{0-1}}^{\star}$ satisfies $\gamma-$ \fairness for $\gamma < \infty$.

\section{Proof of Theorem \ref{result:fair-erm}}
\label{appendix:thm4}

\subsection{Existence of $(\gamma, \H)$-fair classifiers}

In general, it's not always the case that $\H$ itself contains a classifier that is  $\gamma$-fair w.r.t $\H$. This is the case, however, for $\Delta(\H)$, the class of convex combinations of classifiers from $\H$.

\begin{lemma}
\label{lemma:existence}
For every compact $\H \neq \phi$ and $\gamma \in [0, \infty)$, the class $\Delta(\H)$ always contains a $(\gamma, \H)-$fair classifier.
\end{lemma}

\begin{proof}
Let $h$ be some classifier in $\H$. If $Imb(h) =0$, then $h$ is perfectly balanced and in particular $(\gamma, \H)$-fair. Otherwise assume w.l.o.g that the imbalance is in favor of $T$. Next, let $h'$ be the classifier that \emph{minimizes} the directed loss imbalance (in the direction $T \to S$) in $\H$.\footnote{The assumption that $\H$ is compact guarantees this argmin exists. For example, any finite $\H$, or an infinite class parametrized by a compact space such as $\H_w$ (linear or logistic classifiers with a bounded norm).} Now, there are exactly two cases: (i) $dImb(h') \geq 0$, (ii) $dImb(h') < 0$. In the first case, we claim that $h'$ itself is $(\gamma, \H)$-fair. Indeed, according to our definition, for some other $h'' \in \H$ to disqualify $h'$, $h''$ must strictly improve the directed imbalance -- but by definition no classifier in $\H$  can do that. In the second case, we have that $dImb(h) > 0$ and $dImb(h') < 0$. From the linearity of the directed imbalance, there must exist $a$ such that   $dImb(a \cdot h + (1-a) \cdot h' )=0$. This is a classifier in $\Delta(\H)$ that is perfectly balanced and therefore also $(\gamma, \H)$-fair, which concludes the proof.

\end{proof}

\subsubsection{Approximate Fairness}
For a quantity $P$ that depends on the distribution $\P$, we use $\widehat{P}$ to denote the empirical counterpart of $P$ as calculated w.r.t a fixed sample $D \sim \P$.  Note that our definition of disqualification is not immediately applicable when working with finite samples: for example, even a perfectly balanced classifier, $Imb(h) =0$, will have some level of imbalance on a random sample: $\widehat{Imb}(h) \neq 0$.  The same is true with respect to the loss; e.g., even the most accurate classifier (which should always be $\infty$-fair) might not have the optimal loss on a random sample and may therefore be disqualified. Thus, to guarantee generalization we define the following \emph{approximate} variants of $\gamma$-disqualification and $\gamma$-fairness.

\begin{definition}[Approximate disqualification]
\label{def:approx-disq}
A classifier $h'$ is said to $(\alpha_1, \alpha_2, \gamma)$-disqualify $h$ if
\begin{equation}
\label{eqn:approx-gamma-disq}
     {dLossImb}(h; \ell_B)-{dLossImb}(h'; \ell_B) > \alpha_1 + f_\gamma \br{\max \set{0, \,\, {\ell}_A(h')-{\ell}_A(h) + \alpha_2}}
\end{equation}
\end{definition}


\begin{definition}[Approximate fairness]
\label{def:approx-pareto-fairness}
We say that a classifier $h$ is $(\alpha_1, \alpha_2, \gamma)$-fair w.r.t $\H$ if no classifier in $\H$  $(\alpha_1, \alpha_2, \gamma)$-disqualifies it. 
\end{definition}

For simplicity, we will sometimes simply say that $h$ is $({\alpha}, \gamma)$-fair w.r.t $\H$, where ${\alpha} =(\alpha_1, \alpha_2)$ are the approximation parameters (for imbalance and loss, respectively). Additionally, when Definitions (\ref{def:approx-disq}) and (\ref{def:approx-pareto-fairness}) are computed w.r.t a finite sample $D \sim \P^m$ (as opposed to w.r.t $\P$) itself, then we say that $h'$ \emph{empirically} disqualifies $h$ and that $h$ is \emph{empirically} fair, respectively.

With the definitions of approximate and empirical fairness in place, we  define a notion of uniform convergence of a class $\H$ w.r.t our notion of fairness, as follows.

\begin{definition}[Uniform convergence w.r.t fairness]
 We say that a class $\H$ has the uniform convergence property with respect to approximate-fairness with sample complexity $m_{\varepsilon_{1},\varepsilon_{2},\delta}^{Pareto}$ if for any distribution $\P$, w.p $1-\delta$ over the choice of $D\sim\P^{m}$ for $m\geq m_{\varepsilon_{1},\varepsilon_{2},\delta}^{\gamma-Pareto}$, the following holds simultaneously for every $h \in \H$: there exist $\alpha=(\alpha_{1},\alpha_{2})$ and $\hat{{\alpha}}=(\hat{\alpha}_{1},\hat{\alpha}_{2})$ such that: (i) $h$ is $(\alpha, \gamma)$-fair on the underlying distribution $\P$, (ii) $h$ is empirically $(\hat{\alpha},\gamma)$-fair on $D$; (iii) $\left|\alpha_{i}-\hat{\alpha}_{i}\right|\leq\varepsilon_{i}$ for $i=1,2$. 
\end{definition}

Importantly, we show that standard uniform convergence for loss and imbalance indeed guarantees uniform convergence w.r.t our notion of fairness. This is important, since it will allow us to solve the  \emph{fair risk minimization} on a finite sample, and argue that the results transfer (approximately) to the underlying distribution.

\begin{lemma}
If $\H$ has the uniform convergence property w.r.t loss with sample complexity $m_{\varepsilon,\delta}^{Loss}$ and w.r.t imbalance with sample complexity $m_{\varepsilon,\delta}^{Imb}$, then it has uniform convergence property w.r.t approximate fairness, with sample complexity $m_{\varepsilon_{1},\varepsilon_{2},\delta}^{Pareto}=\max\{m_{\varepsilon_{1}/2,\delta/2}^{Imb},m_{\varepsilon_{2}/2,\delta/2}^{Loss}\}$.
\end{lemma}

\begin{proof}
Suppose $m\geq\max\{m_{\varepsilon_{1}/2,\delta/2}^{Imb},m_{\varepsilon_{2}/2,\delta/2}^{Loss}\}$ and consider some classifier $h\in\H$. Suppose that it is $(\alpha,\gamma)$-fair w.r.t $\H$ on $D\sim\P^{m}$. Let $h'$ be any other classifier in $\H$. Using the uniform convergence assumptions, we have that w.p at least $1-\delta$

\begin{align*}
    Imb(h)-Imb(h')	&\leq2\cdot\varepsilon_{1}/2+\widehat{Imb}(h)-\widehat{Imb}(h') \\
	&\leq\varepsilon_{1}+\alpha_{1}+c\cdot\sqrt{\gamma\cdot\left[\hat{\ell}(h')-\hat{\ell}(h)+\alpha_{2}\right]} \\
	&=\varepsilon_{1}+\alpha_{1}+c\cdot\sqrt{\gamma\cdot\left[\ell(h')-\ell(h)+\alpha_{2}+\varepsilon_{2}\right]}
\end{align*}

So $h$ is $(\alpha',\gamma)$-fair  w.r.t $\H$ on $\P$, where $\alpha'_{i}=\alpha_{i}+\varepsilon_{i}$, which implies the required.
\end{proof}


\subsubsection{Approximate Pareto Frontier of $\H$}

The Empirical Pareto frontier of a class of classifiers $\C$ (w.r.t loss and imbalance in a fixed direction, as measured on a sample $D$) is the collection of all Pareto-efficient classifiers in $\C$; that is, thinking of a classifier $h$ as a two-dimensional point $h=(\varphi, \tau) \in[0,1] \times [-1,1] $ where $\widehat{dImb}(h) = \tau $ and $\widehat{\ell}(h) = \varphi$, these are all the points in $\C$ that correspond to classifiers that 
 achieve the best possible loss (of any classifier in $\C$) without exceeding a specific level of imbalance. We define an $\varepsilon$-approximation to the Empirical Pareto frontier as follows:

\boxedeq{eq:pf}{\quad
{\sf{\widehat{PF}}}(\varepsilon, \C; \, D)  \qquad 
        \set{ h(\tau) \,\, \vert \,\, \tau = -1, -1+\varepsilon, \dots, 1-\varepsilon, 1} \quad }

where 
\begin{equation}
\label{eqn:LossAtImb}
    h(\tau) \in \arg\min_{h \in \C}\widehat{\ell}(h) \quad \text{subject to} \quad \widehat{Imb}(h)\leq \tau
\end{equation}

That is, ${\sf{\widehat{PF}}}(\varepsilon, \C) $ is a collection of at most $2/\varepsilon$  classifiers from $\C$, each of which is optimal for their maximal imbalance level. 
We refer to  ${\sf{\widehat{PF}}}(\varepsilon=0, \C) $ as the \emph{full} Empirical Pareto frontier.

\paragraph{Disqualifying on the Pareto frontier.} Given $\alpha$ and $\gamma$, we say that a classifier is (empirically) disqualified on the $\varepsilon$-Pareto-frontier of $\C$ if there is a classifier in ${\sf{\widehat{PF}}}(\varepsilon, \C) $ that $(\alpha, \gamma)$ (empirically) disqualifies it. Note that by definition, if $h$ is a classifier which no classifier on the \emph{full} Pareto frontier disqualifies, then $h$ is (empirically) fair w.r.t $\C$. When we use   ${\sf{\widehat{PF}}}(\varepsilon, \C) $, however, we incur an an additional factor of $\varepsilon$ in the additive imbalance slack:

\begin{lemma}
\label{lemma:eps-net-disq}
Fix a classifier $p$, parameters $\alpha, \gamma$ and a class $\C$. If $p$ is not $(\alpha_1, \alpha_2, \gamma)$-empirically-disqualified on ${\sf{\widehat{PF}}}(\varepsilon, \C) $, then it is not $(\alpha_1+\varepsilon,\alpha_2, \gamma)$--empirically-disqualified on ${\sf{\widehat{PF}}}(\varepsilon=0, \C) $, the full Pareto frontier.
\end{lemma}

\begin{proof}
Let $h$ be any classifier on the full Pareto frontier, with empirical imbalance $\tau$. Denote by $\tau'\geq\tau$ the closest imbalance on the $\varepsilon$-Pareto frontier, and $h'$ the optimal classifier at this imbalance level. Note that by the definition of  ${\sf{\widehat{PF}}}(\varepsilon, \Delta(H)) $, $\tau\leq\tau'\leq\tau+\varepsilon$ and $\widehat{\ell}(h')\leq\widehat\ell(h)$. Now, we can use the fact that $h'$ is on the $\varepsilon$-Pareto frontier and therefore doesn't empirically-disqualify $p$:

\begin{align*}
    \widehat{Imb}(p)-\widehat{Imb}(h)&=\widehat{Imb}(p)-\tau \\
    &\leq \widehat{Imb}(p)-(\tau'-\varepsilon) \\
    &=\varepsilon+\widehat{Imb}(p)-\widehat{Imb}(h') \\
    &\leq\varepsilon+\alpha_1+f_{\gamma}(\widehat{\ell}(h')-\widehat{\ell}(p) + \alpha_2) \\
    &\leq\varepsilon+\alpha_1+f_{\gamma}(\widehat{\ell}(h)-\widehat{\ell}(p)+\alpha_2)
\end{align*}

From which we conclude that  $h$ doesn't empirically $(\alpha_1+\varepsilon,\alpha_2,\gamma)$-disqualify $p$, as required.
\end{proof}

We will also use the fact that if there is a classifier that is empirically fair, then there is also one that is similarly empirically fair \emph{and} is on the approximate Pareto frontier.

\begin{lemma}
\label{lemma:approx-fair-on-pf}
If $p$ is empirically $(\alpha_1, \alpha_2,\gamma)$-fair w.r.t $\H$, then there is a classifier $p'$ on the $\varepsilon$-Pareto frontier oh $\H$ which is $(\alpha_1-\varepsilon,\alpha_2,\gamma)$-fair w.r.t $\H$.
\end{lemma}

\begin{proof}
Denote $\tau = \widehat{Imb}(p)$. Let $\tau'\geq \tau$ denote the closest imbalance level that corresponds to an imbalance on the $\varepsilon$-Pareto frontier, and denote $p' = h(\tau') \in {\sf{\widehat{PF}}}(\varepsilon, \Delta(\H)))$ the classifier that is optimal at this level of imbalance. 
Now, let $h$ be any other classifier - we'll prove it doesn't $(\alpha_1-\varepsilon,\alpha_2,\gamma)$-empirically-disqualify $p'$:

\begin{align*}
    \widehat{Imb}(p')-\widehat{Imb}(h)	&\leq\tau'-\widehat{Imb}(h) \\
	&\leq\tau+\varepsilon-\widehat{Imb}(h) \\
	&=\varepsilon+\widehat{Imb}(p)-\widehat{Imb}(h) \\
	&\leq\varepsilon+\alpha_{1}+f_{\gamma}(\widehat{\ell}(h)-\widehat{\ell}(p)+\alpha_{2}) \\
	&\leq\varepsilon+\alpha_{1}+f_{\gamma}(\widehat{\ell}(h)-\widehat{\ell}(p')+\alpha_{2})
\end{align*}
\end{proof}

\paragraph{Computing the approximate Pareto Frontier.}
We will be working with the approximate Pareto frontier for $\Delta(\H)$, the class of convex combinations of classifiers in $\H$.
In some simple cases, ${\sf{\widehat{PF}}}(\varepsilon, \Delta(\H)) $ can be computed efficiently. For example, when $\H=\H_{w}$ is the class of linear classifiers over $\R^d$ with bounded norm (in which case $\Delta(\H) = \H$),  $\varphi(\tau)$ can be obtained as the solution to a convex program with $d$ variables (since the objective is convex in $w$, and the imbalance constraints are linear in $w$). Therefore in this case ${\sf{\widehat{PF}}}(\varepsilon, \Delta(H)) $ can be computed in time $\poly(1/\varepsilon, d, \card{D})$.

\subsection{Learning optimal fair classifiers using the Pareto Frontier}
\label{section:learning}

A natural objective is to find the \emph{optimal} fair classifier in $\Delta(\H)$ (which always contains a $\gamma$-fair classifier). That is,  given a class $\H$  parameters $\alpha, \gamma$ and a sample $D$, find the optimal $h \in \Delta(\H)$ that is empirically $(\alpha, \gamma)$-fair w.r.t $\H$.  We refer to this as the Fair-ERM problem:

\boxedeq{eq:approx-fair-erm}{\quad {\sf{FairERM}}(\alpha, \gamma, \H)  \qquad 
        \min_{h \in \Delta(\H)} \widehat{\ell}(h) \,\, \text{s.t} \,\, h \,\, \text{is empirically} \,\, (\alpha, \gamma)- \text{fair w.r.t} \,\, \H \quad }


The next proposition proves that an approximation to the Pareto frontier can be used to efficiently obtain an approximation to the FairERM problem.

\begin{proposition}
\label{prop:fair-erm}
Fix a dataset $D$, a class $\H$ and parameters $\alpha_1, \alpha_2, \gamma$. Then, for every $\varepsilon \leq \alpha_1$ the following is true: 
Given oracle access to ${\sf{\widehat{PF}}}(\varepsilon, \Delta(H)) $, there is an efficient algorithm $A$ for finding a classifier $h(A)$ such that: (1) $h(A)$ is empirically $(\alpha_1, \alpha_2, \gamma)$-fair w.r.t $\H$, and (2) $\widehat{\ell}(h(A)) \leq \widehat{\ell}(h)$, where $h$ is any classifier that is empirically $(\alpha_1-\varepsilon, \alpha_2, \gamma)$-fair w.r.t $\H$.

\end{proposition}

Note that by using the  approximation to the Pareto frontier we incur a  degradation in the accuracy guarantee: we are only outputting a classifier whose accuracy is competitive with the optimal $(\alpha_1 -\varepsilon, \alpha_2, \gamma)$-fair classifier (as opposed to with the optimal $(\alpha_1, \alpha_2, \gamma)$-fair classifier). Without assuming anything about $\H$, this accuracy gap could be substantial: in principle, $\H$ could contain two classifiers which differ only minimally in imbalance but significantly in accuracy. This highlights that the strength of the guarantee is related to the \emph{Lipschitzness} of the function $\varphi(\tau; \Delta(\H))$  that returns the optimal loss in $\Delta(\H)$ at a given level of imbalance (see Equation \ref{eqn:LossAtImb}).

\begin{algorithm}
\LinesNumbered
\DontPrintSemicolon
\begin{lrbox}{0}
\begin{minipage}{\hsize}

\vspace{1em}

  
  
  

    Initialize FairClassifiers = [ \, ]  \;
    
  \For{$h \in {\sf{\widehat{PF}}}(\varepsilon, \Delta(H))$} 
  {
  $fair = \mathtt{True}$ \;
  \For{$h' \in {\sf{\widehat{PF}}}(\varepsilon, \Delta(H))$} 
  {
  \If{$h'$ empirically  $(\alpha_1 - \varepsilon, \alpha_2, \gamma)$-disqualifies $h$}
  {
  $fair = \mathtt{False}$ \;
  break \;
  }
  }
  \If{$fair$}
  {
  FairClassifiers.append($h$) \;
  }

  }
  \If{FairClassifiers $== \emptyset$:}{
  \Return{$\perp$}
  }
  
  $h(A) \leftarrow$ the most accurate classifier in FairClassifiers \;
  
  \Return{$h(A)$}

\end{minipage}%
\end{lrbox}
\hspace*{-10pt}\framebox[\columnwidth]{\hspace*{15pt}\usebox{0}}
\caption{${\sf{ApproxFairERM}}(\varepsilon, \alpha_1, \alpha_2, \gamma, \H)$ returns the most accurate classifier in the set ${\sf{\widehat{PF}}}(\varepsilon, \Delta(H))$ that is not empirically $(\alpha_1 - \varepsilon, \alpha_2, \gamma)$-disqualified by another classifier in this set. }\label{fig:algo} 
\end{algorithm}

\begin{proof}
Consider the procedure $A \triangleq {\sf{ApproxFairERM}}(\varepsilon, \alpha_1, \alpha_2, \gamma, \H) $ defined in Figure \ref{fig:algo}, which runs in time $O(1/\varepsilon^2)$. First, we argue that if the output of $A$ is $h\neq \perp$, then $h$ is $(\alpha_1, \alpha_2, \gamma)$-empirically-fair. Note that ${\sf{ApproxFairERM}}$ only returns a classifier  which no other classifier on the $\varepsilon$-Pareto frontier $(\alpha_1-\varepsilon, \alpha_2, \gamma)$ disqualifies. Therefore, Lemma \ref{lemma:eps-net-disq} guarantees that no classifier on the full Pareto frontier $(\alpha_1, \alpha_2, \gamma)$ disqualifies it; this, in turn, guarantees the classifier is $(\alpha_1, \alpha_2, \gamma)$-fair. 
Second, we argue that $A$ always returns $h \neq \perp$. In light of the above, it's sufficient to argue that there always exists a classifier on the $\varepsilon$-Pareto-frontier that is $(\alpha_1, \alpha_2, \gamma)$-fair. This follows by combining Lemma \ref{lemma:approx-fair-on-pf} with the fact that there exists a classifier (not necessarily on the approximate Pareto frontier) that is $(\alpha_1 +\varepsilon, \alpha_2, \gamma)$-fair.

It's left to prove that $\ell(h(A)) \leq \ell(h^\star)$, where $h^\star$ is the optimal $(\alpha_1-\varepsilon, \alpha_2, \gamma)$-empirically-fair classifier.

 Let $h$ denote the closest  classifier (from above) to $h^{\star}$ on ${\sf{PF}}(\varepsilon, \Delta(H))$, so that $\widehat{\ell}(h)\leq\widehat{\ell}(h^{\star})$ and $\widehat{Imb}(h)\leq \widehat{Imb}(h^{\star})+\varepsilon$. We claim that $h$ is empirically $(\alpha_1, \alpha_2, \gamma)$-fair. To see this, let $h'$ be some other classifier. Using the fact that $h^\star$ is $(\alpha_1-\varepsilon, \alpha_2, \gamma)$-fair, we have:

\begin{align*}
\widehat{Imb}(h)-\widehat{Imb}(h')	&\leq\varepsilon+\widehat{Imb}(h^{\star})-\widehat{Imb}(h') \\
	&\leq\varepsilon+\alpha_1-\varepsilon+f_{\gamma}(\widehat{\ell}(h')-\widehat{\ell}(h^{\star})+\alpha_2) \\
	&\leq\alpha_1+f_{\gamma}(\widehat{\ell}(h')-\widehat{\ell}(h^{\star})+\alpha_2) \\
	&\leq\alpha_1+f_{\gamma}(\widehat{\ell}(h')-\widehat{\ell}(h)+\alpha_2)
\end{align*}

So  $h$ is indeed $(\alpha_1, \alpha_2,\gamma)$-fair. Next, we note that by definition, $h(A)$ is the optimal classifier on the $\varepsilon$-Pareto-frontier that is $(\alpha_1,\alpha_2, \gamma)$-fair (this follows by the definition of {\sf{ApproxFairERM}}, and (i)).
Since $h$ is, by construction, also on the $\varepsilon$-Pareto frontier, the fact it is $(\alpha_1, \alpha_2, \gamma)$ fair thus implies that $\widehat{\ell}(h(A))\leq\widehat{\ell}(h)$. Since by definition $\widehat{\ell}(h) \leq \widehat{\ell}(h^\star)$, we have that $\widehat{\ell}(h(A))\leq\widehat{\ell}(h)\leq \widehat{\ell}(h^{\star})$, as required.

\end{proof}

\end{document}

%% file: abstract.tex
\begin{abstract}

In many machine learning settings there is an inherent tension between fairness and accuracy desiderata. How should one proceed in light of such trade-offs? 
In this work we introduce and study \emph{$\gamma$-disqualification}, a new framework for reasoning about fairness-accuracy tradeoffs w.r.t a benchmark class $\H$ in the context of supervised learning. Our requirement stipulates that a classifier should be disqualified if it is possible to improve its fairness by switching to another classifier from $\H$ without paying ``too much'' in accuracy. The notion of ``too much`` is quantified via a parameter $\gamma$ that serves as a
vehicle for specifying acceptable tradeoffs between accuracy and fairness, in a way that is independent from the specific metrics used to quantify fairness and accuracy in a given task. Towards this objective, we establish principled translations between units of accuracy and units of (un)fairness for different accuracy measures. We show $\gamma$-disqualification can be used to easily compare different learning strategies in terms of how they trade-off fairness and accuracy, and we give an efficient reduction from the problem of finding the optimal classifier that satisfies our requirement to the problem of approximating the Pareto frontier of $\H$.

\end{abstract}